\documentclass[11pt]{article}

\usepackage[preprint]{acl}

\usepackage{times}
\usepackage{latexsym}

\usepackage[T1]{fontenc}

\usepackage[utf8]{inputenc}

\usepackage{microtype}

\usepackage{graphicx}

\usepackage{amsmath}
\usepackage{amssymb}
\usepackage{mathtools}
\usepackage{amsthm}
\usepackage{booktabs}
\usepackage{multirow}
\usepackage{arydshln}
\usepackage{subcaption}   
\usepackage{float}
\usepackage{xspace}

\floatstyle{ruled}
\newfloat{algorithm}{tbp}{loa}
\floatname{algorithm}{Algorithm}

\definecolor{LightCyan}{rgb}{0.88,1,1}
\definecolor{LightRed}{rgb}{1,0.85,0.85}
\definecolor{LightYellow}{rgb}{1,1,0.78}
\definecolor{Grey}{rgb}{0.8,0.8,0.8}
\definecolor{DarkGrey}{rgb}{0.5,0.5,0.5}
\definecolor{LightGreen}{rgb}{0.0,0.70,0.0}

\usepackage[capitalize,noabbrev]{cleveref}

\theoremstyle{plain}
\newtheorem{proposition}{Proposition}
\theoremstyle{definition}
\newtheorem{remark}{Remark}

\newcommand{\sysname}{ASRD\xspace}

\title{Follow the Latent Roadmap: Navigating Revocable Decoding for Diffusion LLMs with Anchor Tokens}

\author{
  {\bf Yizhen Yao}$^{1}$\thanks{\, The two authors contributed equally in this project.}, {\bf Qinglin Zhu}$^{1}$\footnotemark[\value{footnote}], {\bf Runcong Zhao}$^{1}$, {\bf Xiangxiang Dai}$^{2}$ \\
  {\bf Yanzheng Xiang}$^{1}$, {\bf Yulan He}$^{1,3}$, {\bf Lin Gui}$^{1}$ \\
  $^{1}$King's College London \quad $^{2}$The Chinese University of Hong Kong \\
  $^{3}$The Alan Turing Institute, UK \\
  \texttt{\{yizhen.yao, qinglin.1.zhu, runcong.zhao\}@kcl.ac.uk} \\
  \texttt{xxdai23@cse.cuhk.edu.hk} \\
  \texttt{\{yanzheng.xiang, yulan.he, lin.1.gui\}@kcl.ac.uk}
}

\begin{document}
\maketitle
\begin{abstract}
Diffusion Large Language Models (dLLMs) offer a promising avenue for parallel generation but face a trade-off between decoding speed and quality. While revocable decoding strategies attempt to mitigate errors by verifying and remasking tokens, they typically operate within a mixed-quality context. This leads to two critical failures: \textit{Error Propagation}, where new tokens absorb erroneous information from unreliable context, and \textit{Local Error Reinforcement}, where errors mutually reinforce each other to evade detection. To alleviate these challenges, we propose ASRD (Anchor-\mbox{Supervised} \mbox{Revocable} Decoding), a training-free framework that operates within the embedding space. ASRD explicitly decouples the decoding context into trusted \textit{Anchor Tokens}, which are identified via temporal consistency, and uncertain candidates. Leveraging a dynamic Anchor Token Cache, we introduce two complementary mechanisms: (1) Anchor-Guided Generation, which injects entropy-weighted anchor signals into masked positions to implicitly rectify attention toward the reliable global skeleton; and (2) Anchor-Perturbed Verification, which applies orthogonal perturbations to uncertain candidate tokens, destabilizing and remasking errors driven by fragile local consensus. Extensive experiments on math and coding benchmarks demonstrate that ASRD outperforms recent remasking baselines, achieving accuracy improvements of up to 6.4\% while accelerating inference throughput by up to 7.2$\times$. The code is available at \url{https://github.com/preordinary/ASRD}.
\end{abstract}

\section{Introduction}\label{sec:intro}

\begin{figure}[ht!]
    \centering
    \includegraphics[width=1\linewidth]{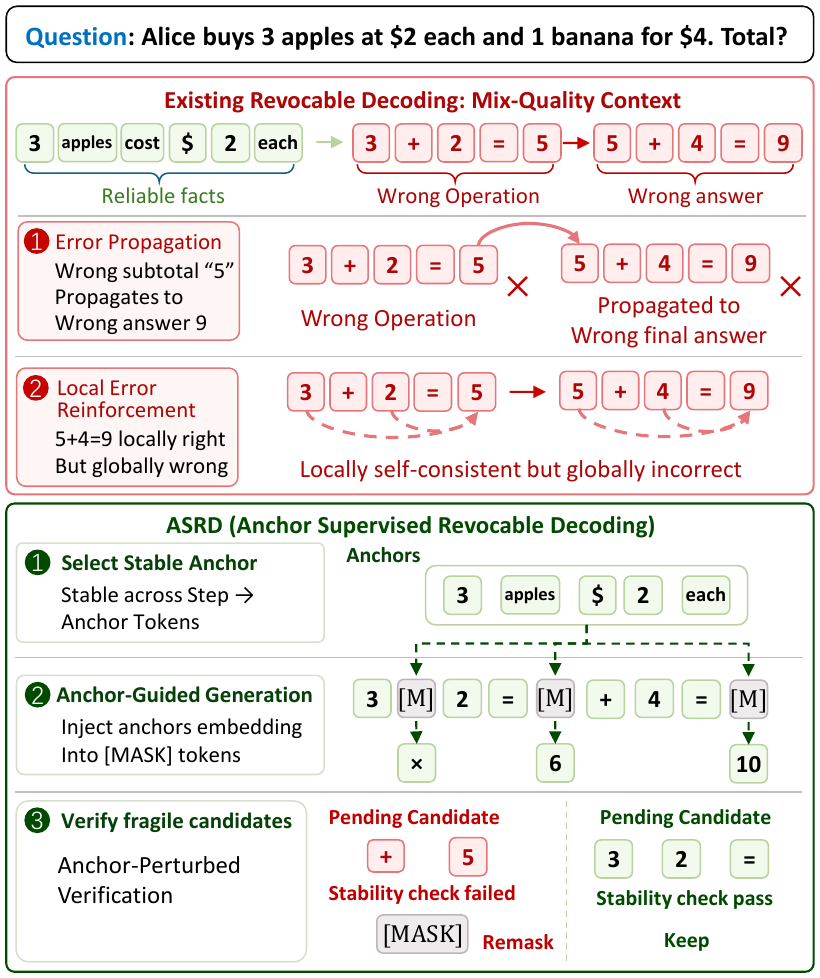}
\caption{
Overview of \sysname{} and its motivation. 
Existing revocable decoding verifies and generates tokens under a mixed-quality context, causing error propagation and local error reinforcement. 
\sysname{} instead selects temporally stable tokens as anchors, injects anchor guidance into masked positions, and verifies pending candidates with anchor-based perturbations. 
}
    \label{fig:sample}
\end{figure}

Diffusion Large Language Models (dLLMs) \citep{nie2025largelanguagediffusionmodels,ye2025dream,song2025seed} decode a sequence by progressively unmasking tokens from a fully-masked initialization, enabling parallel generation of multiple tokens per step rather than the sequential decoding of autoregressive models \citep{achiam2023gpt,grattafiori2024llama3herdmodels,kang2025parallelbench,li2025survey}. However, current threshold-based unmasking strategies \citep{wu2025fastdllmtrainingfreeaccelerationdiffusion,chen2025confidenceadaptivecoherentdecoding,hong2025wide} face a trade-off between speed and quality: a lower threshold commits more tokens per step but admits low-confidence predictions whose errors propagate through later steps \citep{azangulov2025parallelsamplingmaskeddiffusion}. Revocable decoding \citep{hong2025wide, dong2025saber, wang2025remasking} mitigates this by verifying committed tokens within the evolving sequence and \textit{remasking} those judged unreliable. To preserve parallelism, verification and generation share a single forward pass \citep{hong2025wide}, which forces every step to operate on a \textit{mixed-quality context} that interleaves reliable and tentative commitments.


This mixed-quality context compromises both the generation and verification phases of revocable decoding, leading to two critical challenges, as shown in Figure~\ref{fig:sample}:
(1) \textbf{Error Propagation} \citep{azangulov2025parallelsamplingmaskeddiffusion, bie2025llada2}: $\texttt{[MASK]}$ tokens are allowed to freely attend to the pending tokens during generation. This, however, hampers the new tokens' quality, since the generating tokens indiscriminately treat all pending tokens as reliable context, absorbing the potential errors from the pending tokens. Consequently, errors are propagated through the entire generation process, leading to incorrect final outputs. (2) \textbf{Local Error Reinforcement}: Pending tokens are also verified while surrounded by the potential errors from other pending tokens. These erroneous tokens reinforce each other into a locally consistent but globally incorrect cluster.

Addressing the above two failure modes requires two steps: first identifying which committed tokens are trustworthy, then using them to steer the remaining unconfirmed positions. We propose \textbf{\sysname{}} (\underline{A}nchor-\mbox{\underline{S}upervised} \mbox{\underline{R}evocable} \underline{D}ecoding), which carries out both steps in the embedding space, preserving compatibility with FlashAttention~\citep{dao2022flashattention}. The core of ASRD is to promote temporally consistent tokens into a dynamic \textbf{Anchor Token Cache (ATC)}, and to use the resulting trusted skeleton to steer every unconfirmed position. Based on ATC, we propose two complementary updates: (1) \textbf{Anchor-Guided Generation} injects an entropy-weighted anchor signal into mask embeddings to mitigate Error Propagation; (2) \textbf{Anchor-Perturbed Verification} probes committed tokens with an orthogonal anchor-derived signal, flipping only those sustained by Local Error Reinforcement so that they can be remasked for regeneration.

Our contributions are threefold: (1) We propose \sysname{}, a training-free revocable decoding paradigm whose intervention is confined to the embedding space, supplying anchor-derived guidance to every unconfirmed position. (2) We introduce two embedding-level strategies to handle mixed-quality contexts. (3) Serving as a \textit{plug-and-play enhancement} for existing dLLMs, ASRD is a training-free strategy that boosts performance on math and coding.

\section{Related Work}

\textbf{Diffusion Large Language Models (dLLMs).}
To overcome the sequential generation bottleneck of autoregressive LLMs, diffusion LLMs (dLLMs) have emerged as a promising alternative with potential parallel generation.
Early attempts span both continuous~\citep{mahabadi2024tess,strudel2022self,gong2022diffuseq,li2022diffusion} and discrete formulations~\citep{he2023diffusionbert,sahoo2024simple,austin2021structured,ou2024your,lou2024discrete,shi2024simplified}.
Among them, masked diffusion models (MDMs) have shown the most convincing scalability~\citep{gong2024scaling}: they operate directly in token space and iteratively denoise by predicting masked tokens for a fixed number of steps.
Recent progress includes strong open-source models such as LLaDA~\citep{nie2025largelanguagediffusionmodels} and Dream~\citep{ye2025dream}, as well as extensions to coding~\citep{gong2025diffucoder,xie2025dream},
and reasoning~\citep{zhao2025d1,wang2025revolutionizing,zhu2025latentrefinementdecodingenhancing}.
Despite this progress, dLLMs still face a gap between theoretical parallelism and practical performance, due to both expensive bidirectional attention and error accumulation under aggressive parallel unmasking.

\textbf{Efficient Inference for dLLMs.}
Existing efforts to accelerate dLLMs mainly fall into two directions: (i) reducing attention cost via approximate/reusable KV computation~\citep{wu2025fastdllmtrainingfreeaccelerationdiffusion,liu2025dllmcacheacceleratingdiffusionlarge,song2025sparse,ma2025dkv}, and (ii) improving parallel decoding via dynamic thresholding~\citep{israel2025accelerating,wang2025diffusionllmsfasterthanarinference,xiao2026streamingdllmacceleratingdiffusionllms} or error correction~\citep{wang2025remasking,hong2025wide,dong2025saber,xiang2026stopflipflopcontextpreservingverification}.
Dynamic thresholding adaptively adjusts how many tokens to unmask per step based on confidence, while error correction verifies newly unmasked tokens and remasks low-quality ones, which is generally more robust for complex tasks.
For instance, WINO~\citep{hong2025wide} employs a draft-and-verify mechanism with dual thresholds, and Saber~\citep{dong2025saber} combines confidence-aware acceleration with backtracking and remasking.
However, existing error-correction methods often verify tokens under a \emph{mixed-quality} context that may already contain erroneous tokens.

Instead, our work addresses this issue by introducing an \emph{anchor-based context} that explicitly decouples the context into trusted \emph{anchors} and uncertain \emph{candidates}, breaking local reinforcement and enabling more reliable generation and verification.

\begin{figure*}[ht!]
    \centering
\includegraphics[width=0.9\linewidth]{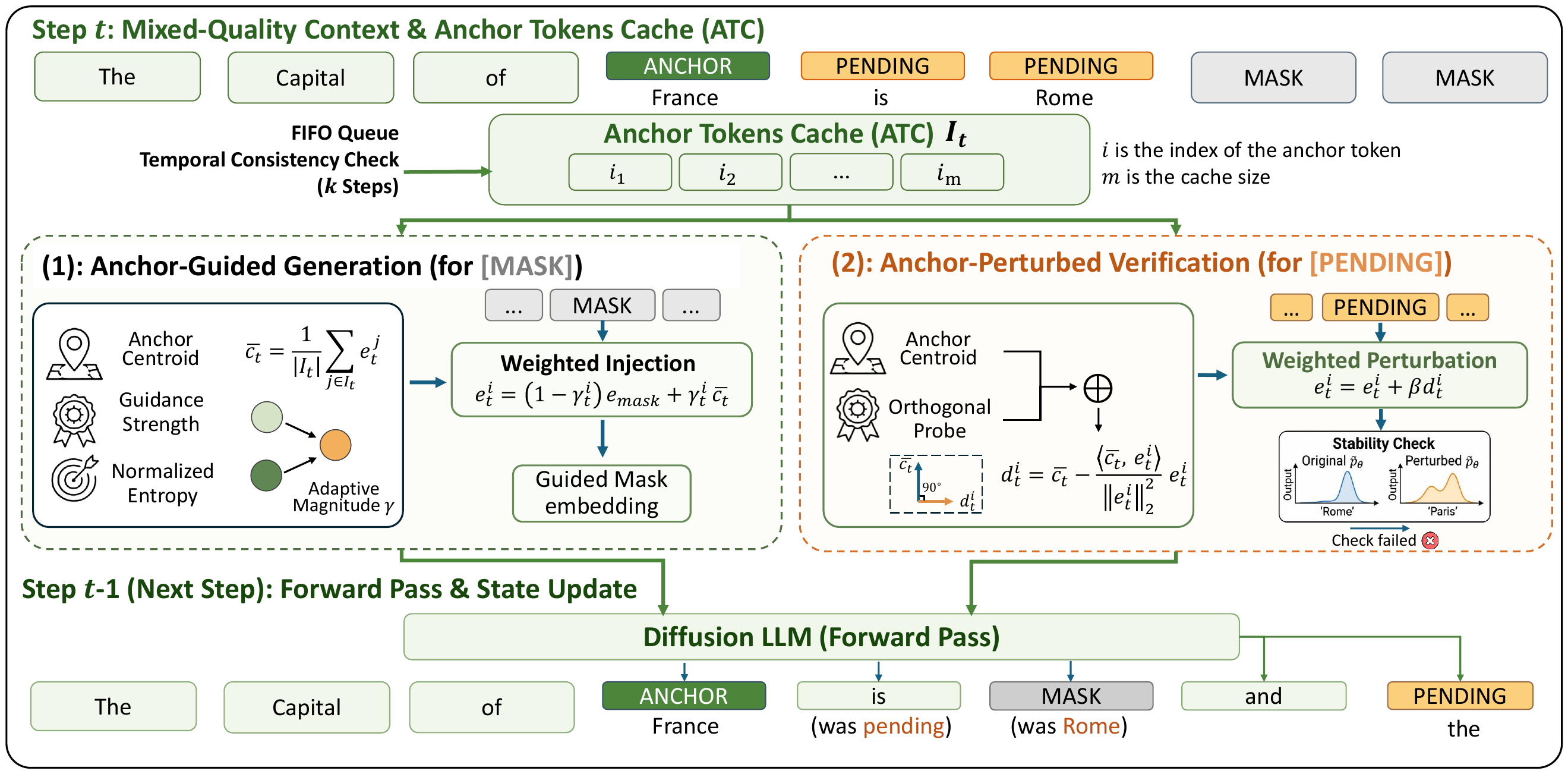}
    \caption{Overview of \sysname{}. At each decoding step, threshold-passed tokens are classified by temporal consistency: those stable across $k$ consecutive steps become \textit{anchors} (stored in the ATC), while the rest are \emph{pending tokens} awaiting verification. The anchor centroid $\bar{c}_t$ then supervises every unconfirmed position through two embedding-space updates. \textbf{Anchor-Guided Generation} (\S\ref{sec:mask_update}) replaces each mask embedding with an entropy-weighted convex blend toward $\bar{c}_t$, injecting a trusted prior before generation and mitigating \emph{Error Propagation}. \textbf{Anchor-Perturbed Verification} (\S\ref{sec:pending_update}) adds a fixed-coefficient probe orthogonal to the committed embedding, stress-testing pending tokens and flipping those sustained by \emph{Local Error Reinforcement} so they can be remasked.}
    \label{fig:frame}
\end{figure*}

\section{Preliminaries: Revocable Parallel Decoding}
\label{sec:prelim_revocable}
For dLLMs \citep{austin2021structured, ou2024your} and a targeting discrete decoding sequence: $\mathbf{x}_0 \in \{1, \dots, V\}^L$, where $V$ is the vocabulary size and $L$ is the sequence length, the forward process progressively corrupts tokens to a special $\texttt{[MASK]}$ state: $q(x_t^i = \texttt{[MASK]} \mid x_{t-1}^i) = \omega_t$, where each token at the $i$-th position, $i\in\{0,1,...,L-1\}$, is independently masked with probability $\omega_t$ at step $t$, and $t\in\{1,2,...,T\}$. 
Accordingly, the reverse process $p_\theta(\mathbf{x}_{t}|\mathbf{x}_{t+1})$ with parameter $\theta$ aims to iteratively unmask the sequence from $\mathbf{x}_T$ (fully masked) to $\mathbf{x}_0$ (clean text).
\begin{equation*}
\small
\underbrace{
x_0 \xrightarrow{\,q\,}x_1 \xrightarrow{\,q\,} \cdots \xrightarrow{\,q\,} 
}_{\text{forward masking (noising)}}
x_T 
\underbrace{
\xrightarrow{\,p_\theta\,} x_{T-1} \xrightarrow{\,p_\theta\,} \cdots \xrightarrow{\,p_\theta\,} x_0
}_{\text{reverse unmasking (denoising)}}
\end{equation*}

In threshold-based parallel decoding, at step $t$, based on the given state in previous step $t+1$, the model predicts $p_\theta(x_t^i \mid \mathbf{x}_{t+1})$ for all masked positions $i$ simultaneously. A token is unmasked and finalized if its confidence exceeds a pre-defined threshold $\tau$.


Revocable parallel decoding adds a verification-and-revision mechanism atop threshold decoding. At each decoding step $t$, the procedure consists of three phases: drafting, verification, and revision.
Let $\mathcal{M}_t = \{i : x_t^i = \texttt{[MASK]}\}$ denote the set of all mask positions at step $t$. 

\textbf{Drafting.} Select positions whose confidence exceeds $\tau$: $\mathcal{T}_t = \{i \in \mathcal{M}_{t+1} : \max_{v} p_\theta(x_0^i{=}v \mid \mathbf{x}_{t+1}) > \tau\}$
and assign each $i \in \mathcal{T}_t$ a candidate token $\breve{x}_t^i = \arg\max_{v} p_\theta(x_0^i{=}v \mid \mathbf{x}_{t+1})$.

\textbf{Verification.} For each $i \in \mathcal{T}_t$, a masked context $\mathbf{x}_t^{\setminus i}$ is built by re-masking position $i$ while keeping the other drafts in place; a forward pass yields the re-evaluated distribution $p_\theta(x_0^i \mid \mathbf{x}_t^{\setminus i})$, and a criterion $\mathcal{F}$ decides whether the draft survives.

\textbf{Revision.} After verification, the set of failed positions is denoted as $\mathcal{R}_t = \{i \in \mathcal{T}_t : \mathcal{F}(i) = \text{fail}\}$. Common revision rules include: \emph{Top-1 Stability} (position $i$ fails if $\arg\max_v p_\theta(x_{t-1}^i{=}v \mid \mathbf{x}_t^{\setminus i}) \neq \breve{x}_t^i$); \emph{Confidence threshold} (position $i$ fails if $\max_v p_\theta(x_{t-1}^i{=}v \mid \mathbf{x}_t^{\setminus i}) < \tau'$ for some threshold $\tau'$), and \emph{Margin criterion} (position $i$ fails if the probability margin decreases significantly). Failed positions are revoked by resetting them to the masked state. 

\section{Methodology}
\label{sec:methodology}
To address the two failure modes identified in Section~\ref{sec:intro}---error propagation during generation and local error reinforcement during verification---we propose \sysname{}. Rather than letting generation and verification share an unfiltered mixed-quality context, we maintain a set of anchors that supervises both. As shown in Figure~\ref{fig:frame}, at each step we use temporal consistency to partition newly generated token positions into anchor positions ($\mathcal{A}_t$) and pending positions ($\mathcal{P}_t$). Anchors are admitted into the Anchor Token Cache (ATC) that forms the trusted backbone of the sequence, and in the next step they supervise the generation of $\texttt{[MASK]}$ tokens via an entropy-weighted injection that suppresses error propagation, as well as the verification of pending tokens via a direction-specific perturbation that exposes local error reinforcement. Algorithm~\ref{alg:asrd} summarizes the complete decoding procedure.

\begin{algorithm}[t]
\caption{\sysname{} decoding with one model forward pass per step.}
\label{alg:asrd}
\small
\renewcommand{\arraystretch}{1.08}
\begin{tabular}{@{}r@{\hspace{0.5em}}p{0.84\columnwidth}@{}}
\multicolumn{2}{@{}p{0.98\columnwidth}@{}}{\textbf{Input:} Denoiser $p_\theta$; length $L$; steps $T$; threshold $\tau$; window $k$; cache capacity $m$; strengths $\alpha,\beta$.} \\
\multicolumn{2}{@{}p{0.98\columnwidth}@{}}{\textbf{Output:} Decoded sequence $\mathbf{x}$.} \\
1 & Initialize $\mathbf{x}\leftarrow[\texttt{MASK}]^L$, $\mathcal{I},\mathcal{P}\leftarrow\varnothing$, $\mathcal{H}^i\leftarrow()$, and $\bar{E}_0^i\leftarrow0$. \\
2 & \textbf{for} decoding step $s=1,\ldots,T$ \textbf{do} \\
3 & $\quad \mathcal{M}\leftarrow\{i:x^i=\texttt{MASK}\}$ and $\mathbf{E}\leftarrow\operatorname{Embed}(\mathbf{x})$. \\
4 & $\quad$\textbf{if} $\mathcal{I}\neq\varnothing$ \textbf{then} \\
5 & $\qquad \bar{c}\leftarrow |\mathcal{I}|^{-1}\sum_{j\in\mathcal{I}}\mathbf{e}^j$. \\
6 & $\qquad$\textbf{for each} $i\in\mathcal{M}$, set $\gamma^i\leftarrow\alpha\bar{E}_{s-1}^i$ and \\
  & $\qquad \mathbf{e}^i\leftarrow(1-\gamma^i)\mathbf{e}_{\mathrm{mask}}+\gamma^i\bar{c}$. \\
7 & $\qquad$\textbf{for each} $i\in\mathcal{P}$, set \\
  & $\qquad \mathbf{d}^i\leftarrow\bar{c}-\frac{\langle\bar{c},\mathbf{e}^i\rangle}{\|\mathbf{e}^i\|_2^2}\mathbf{e}^i$ and $\mathbf{e}^i\leftarrow\mathbf{e}^i+\beta\mathbf{d}^i$. \\
8 & $\quad$\textbf{end if} \\
9 & $\quad p_s\leftarrow p_\theta(\mathbf{E})$. \hfill $\triangleright$ One model forward pass \\
10 & $\quad \mathcal{R}\leftarrow\{i\in\mathcal{P}:\arg\max_v p_s^i(v)\neq\breve{x}^i\}$; \\
   & $\quad$set $x^i\leftarrow\texttt{MASK}$ for every $i\in\mathcal{R}$. \\
11 & $\quad \mathcal{T}_s\leftarrow\{i\in\mathcal{M}:\max_v p_s^i(v)>\tau\}$; \\
   & $\quad$set $\breve{x}^i\leftarrow\arg\max_v p_s^i(v)$ and $x^i\leftarrow\breve{x}^i$ for each $i\in\mathcal{T}_s$. \\
12 & $\quad$Append $\arg\max_v p_s^i(v)$ to $\mathcal{H}^i$ for each $i\in\mathcal{M}$; \\
   & $\quad \mathcal{A}_s\leftarrow\{i\in\mathcal{T}_s:\operatorname{Consistent}_k(\mathcal{H}^i)\}$. \\
13 & $\quad$\textbf{if} $s<k$ and $\mathcal{T}_s\neq\varnothing$ \textbf{then} \\
   & $\quad \mathcal{A}_s\leftarrow\{\arg\min_{i\in\mathcal{T}_s} H(p_s^i)\}$. \hfill $\triangleright$ Seed cache \\
14 & $\quad \mathcal{I}\leftarrow\operatorname{FIFO}_m(\mathcal{I}\cup\mathcal{A}_s)$ and $\mathcal{P}\leftarrow\mathcal{T}_s\setminus\mathcal{A}_s$. \\
15 & $\quad \bar{E}_s\leftarrow\operatorname{MinMaxNormalize}(\{H(p_s^i):i\in\mathcal{M}\})$ for use at the next step. \\
16 & \textbf{end for} \\
17 & \textbf{return} $\mathbf{x}$. \\
\end{tabular}
\end{algorithm}

\subsection{Anchor Token Cache (ATC)}
Not every threshold-passed token should serve as an anchor for downstream decisions: some merely exhibit transient confidence that fluctuates in later steps~\citep{chen2025confidenceadaptivecoherentdecoding, li2025diffusion}. We therefore promote a token to an \emph{anchor} only when its top-1 prediction stays consistent across $k$ consecutive steps. Concretely, let $\hat{a}^{i}_{s} = \arg\max_{v} p_\theta(x_0^i{=}v \mid \mathbf{x}_{s+1})$ denote the top-1 prediction at step $s$. We define the per-step anchor candidates at step $t$ as
\begin{equation*}
\mathcal{A}_t = \left\{ i \in \mathcal{T}_t \;:\; \hat{a}^{i}_{t-k+1} = \hat{a}^{i}_{t-k+2} = \cdots = \hat{a}^{i}_{t} \right\}
\label{eq:anchor_definition}
\end{equation*}

We show that the temporal consistency exponentially suppresses false anchors under idealized assumptions. Moreover, our experiments (Table~\ref{tab:k_ablation}) empirically show rapid drop in error rate. 


\begin{proposition}[Exponential Suppression of False Anchors]
\label{prop:false_anchor}
Assume that (i) there exists a positive semantic margin $\Delta_{\mathrm{sem}} > 0$
between the ground-truth token and any incorrect token $w$ in the ideal distribution,
and (ii) the model prediction can be decomposed as $p_\theta = p^* + \epsilon$,
and for any two candidates the differential noise
$\eta_s = \epsilon_{s, x_0^i} - \epsilon_{s, w}$ is zero-mean,
symmetric around $0$, and i.i.d.\ across decoding steps.
Then, for any position $i$ and any incorrect token $w \neq x_0^i$, the probability that $w$ is the top-1 prediction at position $i$ for $k$ consecutive steps decays exponentially in $k$:
\begin{equation*}
\mathbb{P}\!\left(\hat{a}^{i}_{t-k+1} = \cdots = \hat{a}^{i}_{t} = w\right) \le \exp(-\lambda k)
\end{equation*}
for some $\lambda > 0$ determined by the semantic margin and the noise distribution.\footnote{Proof in Appendix~\ref{app:theory}.}
\end{proposition}

\textbf{ATC Management.}
To track the evolving sequence backbone, we maintain a fixed-size Anchor Token Cache (ATC) $\mathcal{I}_t$ of capacity $m$ that accumulates per-step candidates from $\mathcal{A}_t$ under a first-in-first-out (FIFO) policy:
\begin{equation*}
\mathcal{I}_t = \mathrm{FIFO}_m\!\left(\mathcal{I}_{t-1} \cup \{i: i \in \mathcal{A}_t\}\right)
\end{equation*}
where each entry records a position index $i$. This policy aligns with the approximate left-to-right confirmation order of dLLM unmasking~\citep{horvitz2025computeleftbehindrethinking}, so a FIFO cache retains the most recently stabilized region of the backbone.

\subsection{Anchor-Guided Generation: Combating Error Propagation}
\label{sec:mask_update}
A masked position carries no semantic content; its embedding is the generic $\mathbf{e}_{\text{mask}}$, leaving the subsequent forward pass vulnerable to pollution from erroneous pending neighbors~\citep{azangulov2025parallelsamplingmaskeddiffusion, bie2025llada2}. We counter this by injecting an anchor-derived prior into mask embedding before the forward pass. To this end, we summarize the trusted backbone by the anchor centroid at step $t$,
\begin{equation*}
\bar{c}_t = \frac{1}{|\mathcal{I}_t|}\sum_{j\in\mathcal{I}_t}\mathbf{e}^{j}_t
\end{equation*}
where each $\mathbf{e}^{j}_t$ is the input-layer word embedding of the anchor token committed at position $j$, so the centroid pools the input embeddings of all positions currently held in the ATC and serves as the source of anchor supervision for both the mask and pending updates.

Given this centroid, for each mask position $i \in \mathcal{M}_t$ we replace the uninformative mask embedding by a convex blend toward $\bar{c}_t$:
\begin{equation*}
\mathbf{e}^{i}_t = (1-\gamma^{i}_t)\,\mathbf{e}_{\text{mask}} + \gamma^{i}_t\,\bar{c}_t
\end{equation*}
The mask representation is therefore steered into the trusted subspace before any attention is computed.

To avoid over-correcting reliable predictions, we scale the update magnitude by the normalized prediction entropy of each mask position, taken from the previous step's logits at the same position so that no extra forward pass is required:
\begin{equation*}
\gamma_t^i = \alpha\,\bar{E}_t^i
\end{equation*}
where $\alpha\in[0,1]$ controls the guidance strength and $\bar{E}_t^i\in[0,1]$ is the min--max normalized entropy of position $i$ across all mask positions:
\begin{equation*}
\bar{E}_t^i =
\begin{cases}
\dfrac{E_t^i - E_{\min}}{E_{\max} - E_{\min}}, & \text{if } E_{\max} > E_{\min} \\
0, & \text{otherwise}
\end{cases}
\end{equation*}
Consequently, high-entropy masks receive stronger anchor supervision, while low-entropy masks remain nearly unchanged.

\textbf{Local stability of the embedding intervention.}
Convexity alone does not imply that an interpolated embedding exactly follows the training distribution; our claim is therefore local rather than that the intervention is strictly in-distribution. Denoting the post-intervention embedding by $\mathbf{e}^{i\prime}_t$, the update can be written as
\begin{equation*}
\mathbf{e}^{i\prime}_t = \mathbf{e}_{\text{mask}} + \boldsymbol{\delta}_t^i,
\qquad
\boldsymbol{\delta}_t^i = \gamma_t^i(\bar{c}_t-\mathbf{e}_{\text{mask}}).
\end{equation*}
Since $0\leq\gamma_t^i=\alpha\bar{E}_t^i\leq\alpha$ and $\bar{c}_t$ lies in the convex hull of learned token embeddings, $\mathbf{e}^{i\prime}_t$ stays on the segment between $\mathbf{e}_{\text{mask}}$ and $\bar{c}_t$, with
\begin{equation*}
\|\boldsymbol{\delta}_t^i\|_2
\leq \alpha\|\bar{c}_t-\mathbf{e}_{\text{mask}}\|_2.
\end{equation*}
If the position-$i$ logit map $f_i$ is locally $L_i$-Lipschitz along this segment, then $\|f_i(\mathbf{e}^{i\prime}_t)-f_i(\mathbf{e}_{\text{mask}})\|_2\leq L_i\|\boldsymbol{\delta}_t^i\|_2$. Thus, entropy modulation makes the intervention negligible for low-entropy positions and larger but still bounded for uncertain positions. This is a local stability argument, not a global guarantee of correctness. Empirically, the broad stable region in Figure~\ref{fig:ab_alphabeta}, the module and magnitude-form ablations in Tables~\ref{tab:ablationresult} and~\ref{tab:perturbation_type}, and the attention redistribution in Figure~\ref{fig:combined_analysis} support the prediction quality of the bounded update; deployed values are restricted to $\alpha\in[0.05,0.2]$ (Appendix~\ref{sec:experiments_details}).

\textbf{Implicit attention rectification.} Although the above embedding update modifies embeddings rather than attention scores, the induced query shift biases subsequent attention toward anchor-aligned context. We provide an intuition for this in Appendix~\ref{appen_pro2}.

\subsection{Anchor-Perturbed Verification: Breaking Local Error Reinforcement}
\label{sec:pending_update}
Let the \emph{pending set} be the threshold-passed positions that have not been anchored: $\mathcal{P}_t = \mathcal{T}_t \setminus \mathcal{A}_t$.
A pending token at position $i \in \mathcal{P}_t$ already carries a committed embedding $\mathbf{e}^{i}_t$ whose top-1 prediction we wish to verify. However, this semantic direction may be sustained by \emph{local error reinforcement} with neighboring errors rather than by the global trajectory. Because direct overwriting, as in the mask update, would be destructive, we instead apply a non-destructive \emph{probe} that reuses the anchor centroid $\bar{c}_t$ defined above.

To keep the probe orthogonal to the committed direction, we extract the component of $\bar{c}_t$ orthogonal to $\mathbf{e}^{i}_t$:
\begin{equation*}
\mathbf{d}^{i}_t = \bar{c}_t - \frac{\langle \bar{c}_t,\,\mathbf{e}^{i}_t\rangle}
{\|\mathbf{e}^{i}_t\|_2^2}
\,\mathbf{e}^{i}_t 
\end{equation*}
Since $\mathbf{d}^{i}_t$ is the orthogonal residual of $\bar{c}_t$ with respect to $\mathbf{e}^{i}_t$, its norm is bounded by $\|\bar{c}_t\|_2$. Thus, the probe remains on the same scale regardless of how pending tokens change.

We then add this probe to the committed embedding with a constant coefficient $\beta$, yielding the perturbed embedding
\begin{equation*}
\mathbf{e}_t^i = \mathbf{e}_t^i + \beta\,\mathbf{d}^{i}_t
\end{equation*}
Anchor-aligned tokens induce small residuals and are only weakly perturbed, whereas consensus-sustained errors produce larger orthogonal residuals and receive stronger probes. The perturbed embeddings then enter the next forward pass, and we adopt the Top-1 Stability criterion from Section~\ref{sec:prelim_revocable} as our verification rule: position $i$ fails and is remasked whenever its post-perturbation top-1 prediction differs from the committed draft $\breve{x}_t^i$.

\textbf{Visible-position logits as a fragility signal.}
The denoising objectives of LLaDA and Dream directly supervise masked positions, not already visible positions. We therefore do not interpret visible-position logits as calibrated token predictions. Instead, they provide a counterfactual stability signal: does a committed draft remain supported after a small, structured probe toward the trusted anchor context? By construction,
\begin{equation*}
\langle\mathbf{d}_t^i,\mathbf{e}_t^i\rangle=0,
\qquad
\|\beta\mathbf{d}_t^i\|_2\leq\beta\|\bar{c}_t\|_2,
\end{equation*}
so the probe is bounded and does not directly reinforce the committed-token direction. Let $m_i(\mathbf{E})$ be the committed token's logit margin under the input embedding sequence $\mathbf{E}$, and let $\mathbf{D}_i$ contain $\mathbf{d}_t^i$ only at position $i$. If each relevant logit is locally $L_i$-Lipschitz, then
\begin{equation*}
\big|m_i(\mathbf{E}+\beta\mathbf{D}_i)-m_i(\mathbf{E})\big|
\leq 2L_i\beta\|\mathbf{d}_t^i\|_2.
\end{equation*}
A top-1 flip under this bounded probe therefore identifies a draft with a small local robustness margin. This is deliberately a one-sided fragility test: a flip indicates instability, whereas stability alone is not claimed to certify semantic correctness. Its practical utility is supported by the module and direction-form ablations in Tables~\ref{tab:ablationresult} and~\ref{tab:perturbation_type}; Table~\ref{tab:perturb_metrics} further shows that the probe sharply reduces long contiguous error spans.

\begin{table*}[h]
\centering
\renewcommand{\arraystretch}{1}
\resizebox{1.0 \linewidth}{!}{
\begin{tabular}{lc l lrc lrc lrc lrc}
\toprule[1pt]
\multirow{2}{*}{\textbf{Model}} & \multirow{2}{*}{\textbf{Len}} & \multirow{2}{*}{\textbf{Method}} & \multicolumn{3}{c}{\textbf{HumanEval}} & \multicolumn{3}{c}{\textbf{MBPP}} & \multicolumn{3}{c}{\textbf{GSM8K}} & \multicolumn{3}{c}{\textbf{MATH500}} \\
\cmidrule(lr){4-6} \cmidrule(lr){7-9} \cmidrule(lr){10-12} \cmidrule(lr){13-15}
& & & Acc $\uparrow$ & Steps $\downarrow$ & Speed $\uparrow$ & Acc $\uparrow$ & Steps $\downarrow$ & Speed $\uparrow$ & Acc $\uparrow$ & Steps $\downarrow$ & Speed $\uparrow$ & Acc $\uparrow$ & Steps $\downarrow$ & Speed $\uparrow$ \\
\hline
\multirow{8}{*}{\textbf{LLaDA-Ins-8B}} & \multirow{4}{*}{256} & baseline & \textcolor{DarkGrey}{38.7} & \textcolor{DarkGrey}{256.0} & \textcolor{DarkGrey}{1.0×} & \textcolor{DarkGrey}{36.8} & \textcolor{DarkGrey}{256.0} & \textcolor{DarkGrey}{1.0×} & \textcolor{DarkGrey}{77.4} & \textcolor{DarkGrey}{256.0} & \textcolor{DarkGrey}{1.0×} & \textcolor{DarkGrey}{33.8} & \textcolor{DarkGrey}{256.0} & \textcolor{DarkGrey}{1.0×} \\
& & WINO & 37.5 & 75.8 & 2.3× & 36.2 & 91.1 & 1.9× & 77.3 & \textbf{53.7} & 2.9× & 34.2 & 78.6 & 2.0×\\
& & Saber & 39.4 & 101.2 & 2.4× & 37.8 & 115.4 & 2.1× & 77.8 & 120.5 & 1.9× & 34.8 & 135.2 & 1.8× \\
& & Ours & \textbf{41.5}$_{\textcolor{LightGreen}{+2.8}}$ & \textbf{75.4} & \textbf{3.1×} & \textbf{38.6}$_{\textcolor{LightGreen}{+1.8}}$ & \textbf{88.3} & \textbf{2.5×} & \textbf{79.5}$_{\textcolor{LightGreen}{+2.1}}$ & 56.9 & \textbf{3.2×} & \textbf{35.2}$_{\textcolor{LightGreen}{+1.4}}$ & \textbf{74.8} & \textbf{2.9×} \\
\cmidrule(lr){3-15}

& \multirow{4}{*}{512} & baseline & \textcolor{DarkGrey}{43.9} & \textcolor{DarkGrey}{512.0} & \textcolor{DarkGrey}{1.0×} & \textcolor{DarkGrey}{38.2} & \textcolor{DarkGrey}{512.0} & \textcolor{DarkGrey}{1.0×} & \textcolor{DarkGrey}{80.3} & \textcolor{DarkGrey}{512.0} & \textcolor{DarkGrey}{1.0×} & \textcolor{DarkGrey}{37.8} & \textcolor{DarkGrey}{512.0} & \textcolor{DarkGrey}{1.0×} \\
& & WINO & 43.9 & 133.4 & 2.8× & 37.4 & 133.3 & 2.4× & 78.7 & \textbf{83.1} & 3.6× & 35.6 & 120.6 & 2.6×\\
& & Saber & 44.5 & 198.3 & 2.3× & 37.8 & 188.9 & 2.6× & 79.7 & 141.7 & 3.2× & 37.4 & 192.5 & 2.5× \\
& & Ours & \textbf{48.8}$_{\textcolor{LightGreen}{+4.9}}$ & \textbf{127.8} & \textbf{3.6×} & \textbf{39.4}$_{\textcolor{LightGreen}{+1.2}}$ & \textbf{110.4} & \textbf{2.9×} & \textbf{81.1}$_{\textcolor{LightGreen}{+0.8}}$ & 88.5 & \textbf{4.3×} & \textbf{38.8}$_{\textcolor{LightGreen}{+1.0}}$ & \textbf{119.3} & \textbf{3.8×}\\

\hline
\multirow{8}{*}{\textbf{Dream-Ins-7B}} & \multirow{4}{*}{256} & baseline & \textcolor{DarkGrey}{54.9} & \textcolor{DarkGrey}{256.0} & \textcolor{DarkGrey}{1.0×} & \textcolor{DarkGrey}{57.4} & \textcolor{DarkGrey}{256.0} & \textcolor{DarkGrey}{1.0×} & \textcolor{DarkGrey}{80.4} & \textcolor{DarkGrey}{256.0} & \textcolor{DarkGrey}{1.0×} & \textcolor{DarkGrey}{37.8} & \textcolor{DarkGrey}{256.0} & \textcolor{DarkGrey}{1.0×} \\
& & WINO & 51.3 & 83.8 & 2.4× & 56.8 & 64.7 & 3.6× & 78.6 & 72.7 & 3.4× & 38.4 & 126.6 & 1.6×\\
& & Saber & 52.5 & 165.1 & 1.5× & 57.6 & 191.0 & 1.3× & 79.9 & 185.6 & 1.3× & 38.8 & 179.4 & 1.4× \\
& & Ours & \textbf{56.1}$_{\textcolor{LightGreen}{+1.2}}$ & \textbf{82.6} & \textbf{3.1×} & \textbf{58.8}$_{\textcolor{LightGreen}{+1.4}}$ & \textbf{62.4} & \textbf{4.1×} & \textbf{80.7}$_{\textcolor{LightGreen}{+0.3}}$ & \textbf{65.6} & \textbf{3.9×} & \textbf{41.6}$_{\textcolor{LightGreen}{+3.8}}$ & \textbf{78.1} & \textbf{3.5×} \\
\cmidrule(lr){3-15}

& \multirow{4}{*}{512} & baseline & \textcolor{DarkGrey}{56.1} & \textcolor{DarkGrey}{512.0} & \textcolor{DarkGrey}{1.0×} & \textcolor{DarkGrey}{56.2} & \textcolor{DarkGrey}{512.0} & \textcolor{DarkGrey}{1.0×} & \textcolor{DarkGrey}{80.2} & \textcolor{DarkGrey}{512.0} & \textcolor{DarkGrey}{1.0×} & \textcolor{DarkGrey}{38.6} & \textcolor{DarkGrey}{512.0} & \textcolor{DarkGrey}{1.0×} \\
& & WINO & 51.3 & 102.5 & 3.9× & 55.4 & 86.1 & 5.2× & 79.5 & 92.7 & 5.5× & 39.8 & 139.9 & 2.8×\\
& & Saber & 53.1 & 264.6 & 1.9× & 55.8 & 388.2 & 1.3× & 79.9 & 314.5 & 1.6× & 41.4 & 250.7 & 2.0× \\
& & Ours & \textbf{56.7}$_{\textcolor{LightGreen}{+0.6}}$ & \textbf{85.3} & \textbf{5.5×} & \textbf{57.2}$_{\textcolor{LightGreen}{+1.0}}$ & \textbf{74.5} & \textbf{7.2×} & \textbf{81.1}$_{\textcolor{LightGreen}{+0.9}}$ & \textbf{78.2} & \textbf{6.8×} & \textbf{45.0}$_{\textcolor{LightGreen}{+6.4}}$ & \textbf{94.8} & \textbf{5.4×} \\
\bottomrule[1pt]
\end{tabular}
}
\caption{Performance of different instruct (ins) models and methods across benchmarks.
Speed denotes relative runtime (baseline = 1.0×).
Baseline results are shown in \textcolor{DarkGrey}{grey}, and \sysname improvements in \textcolor{LightGreen}{green}.}
\label{tab:main_result}
\end{table*}

\begin{figure*}[ht]
    \centering
    \begin{subfigure}[b]{0.23\textwidth}
        \centering
        \includegraphics[width=\textwidth]{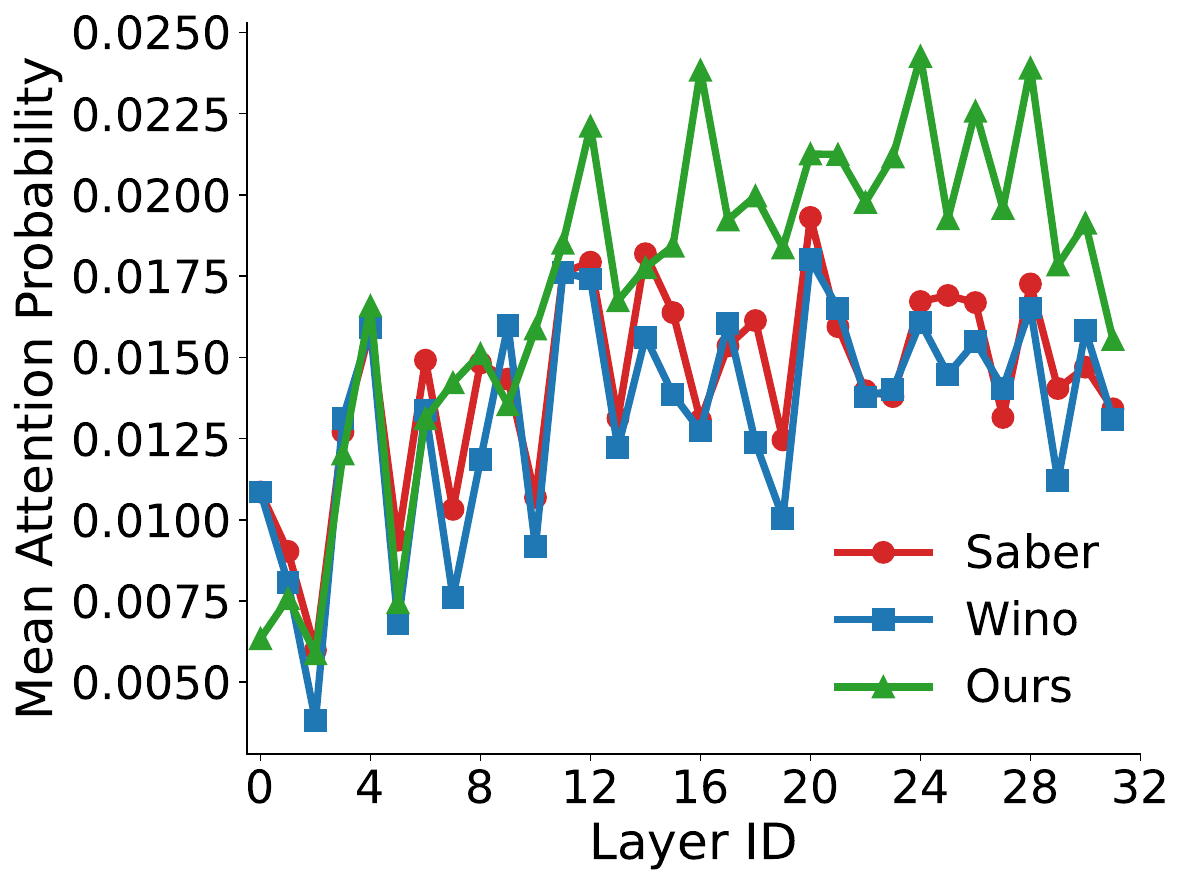}
        \caption{Mask-to-anchor attention: Anchor-Guided Generation increases attention from mask to anchor tokens.}
        \label{fig:effect_2}
    \end{subfigure}
    \hfill
    \begin{subfigure}[b]{0.23\textwidth}
        \centering
        \includegraphics[width=\textwidth]{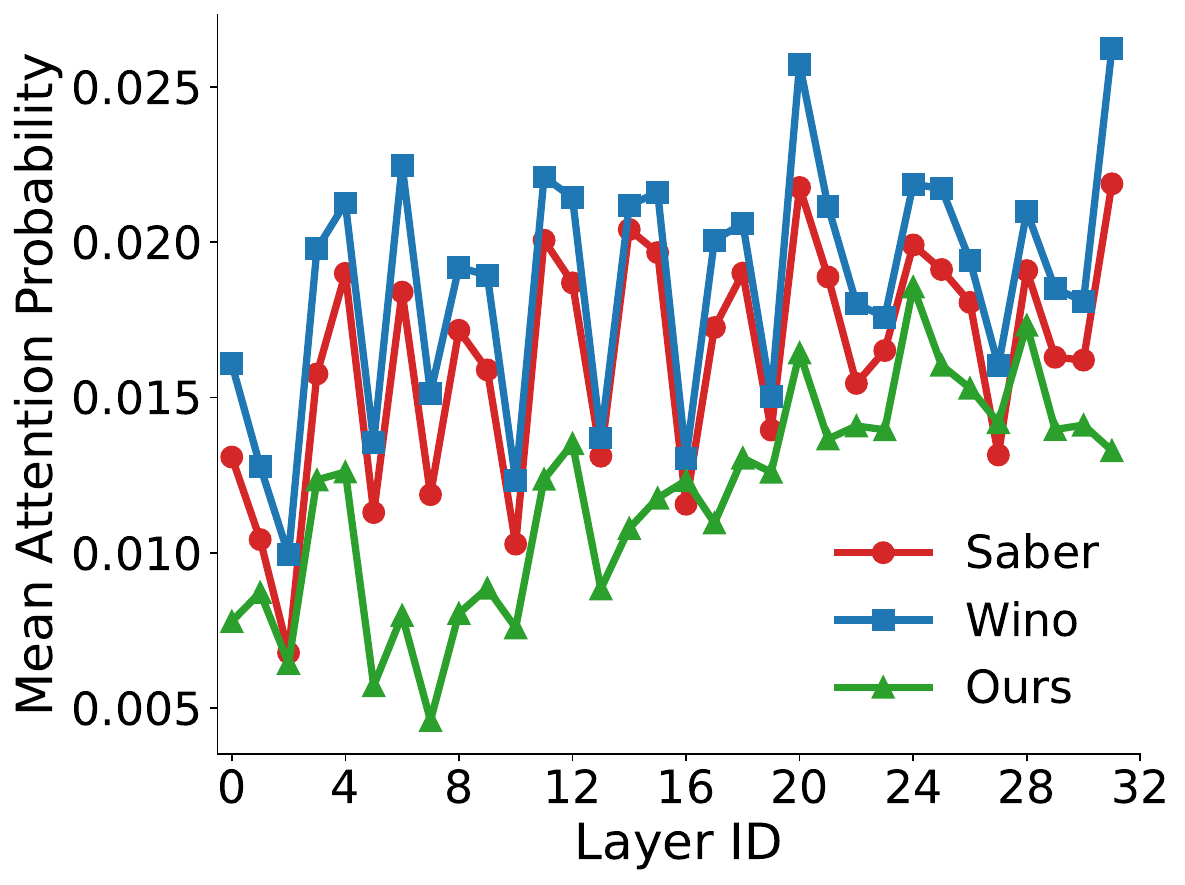}
        \caption{Mask-to-pending attention: Anchor-Guided Generation suppresses attention to uncertain pending tokens.}
        \label{fig:effect_1}
    \end{subfigure}
    \hfill
    \begin{subfigure}[b]{0.23\textwidth}
        \centering
        \includegraphics[width=\textwidth]{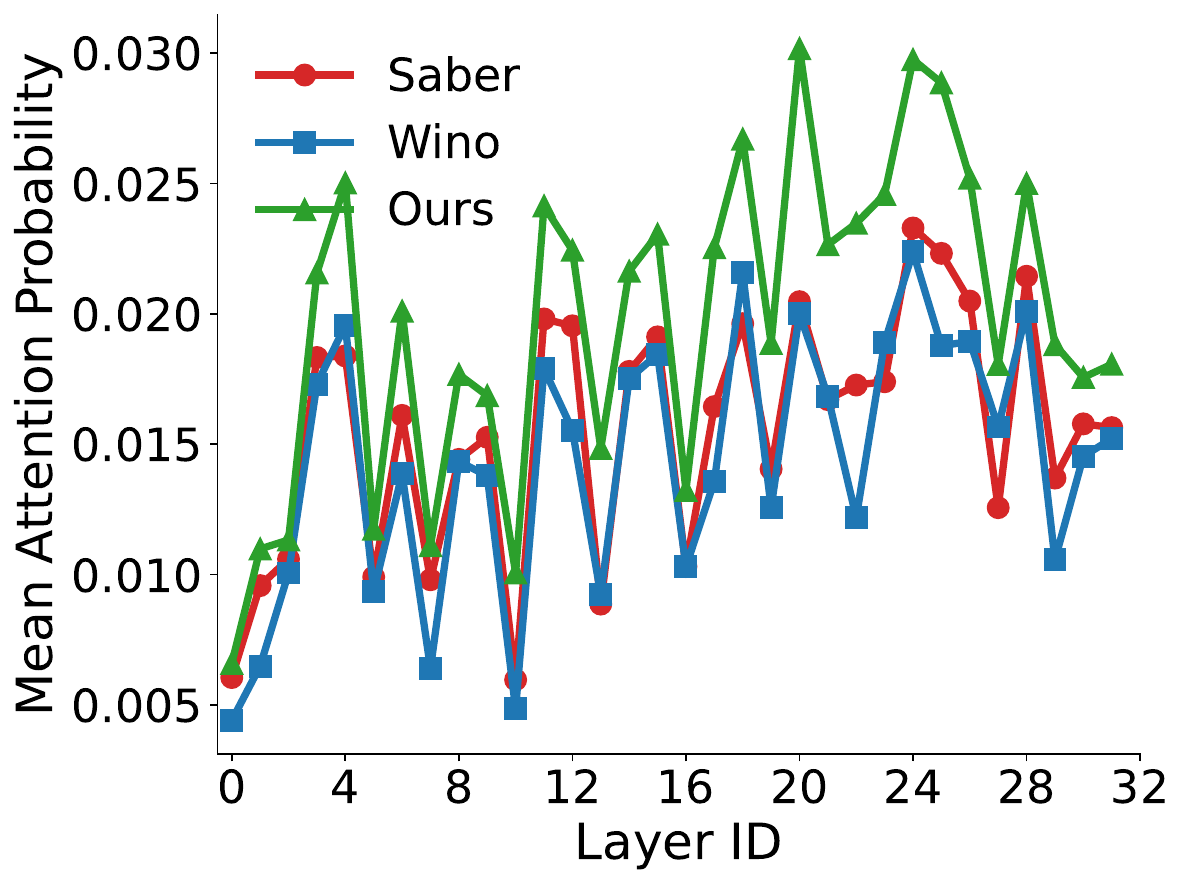}
        \caption{Pending-to-anchor attention: Anchor-Perturbed Verification strengthens alignment with the backbone.}
        \label{fig:effect_4}
    \end{subfigure}
    \hfill
    \begin{subfigure}[b]{0.23\textwidth}
        \centering
        \includegraphics[width=\textwidth]{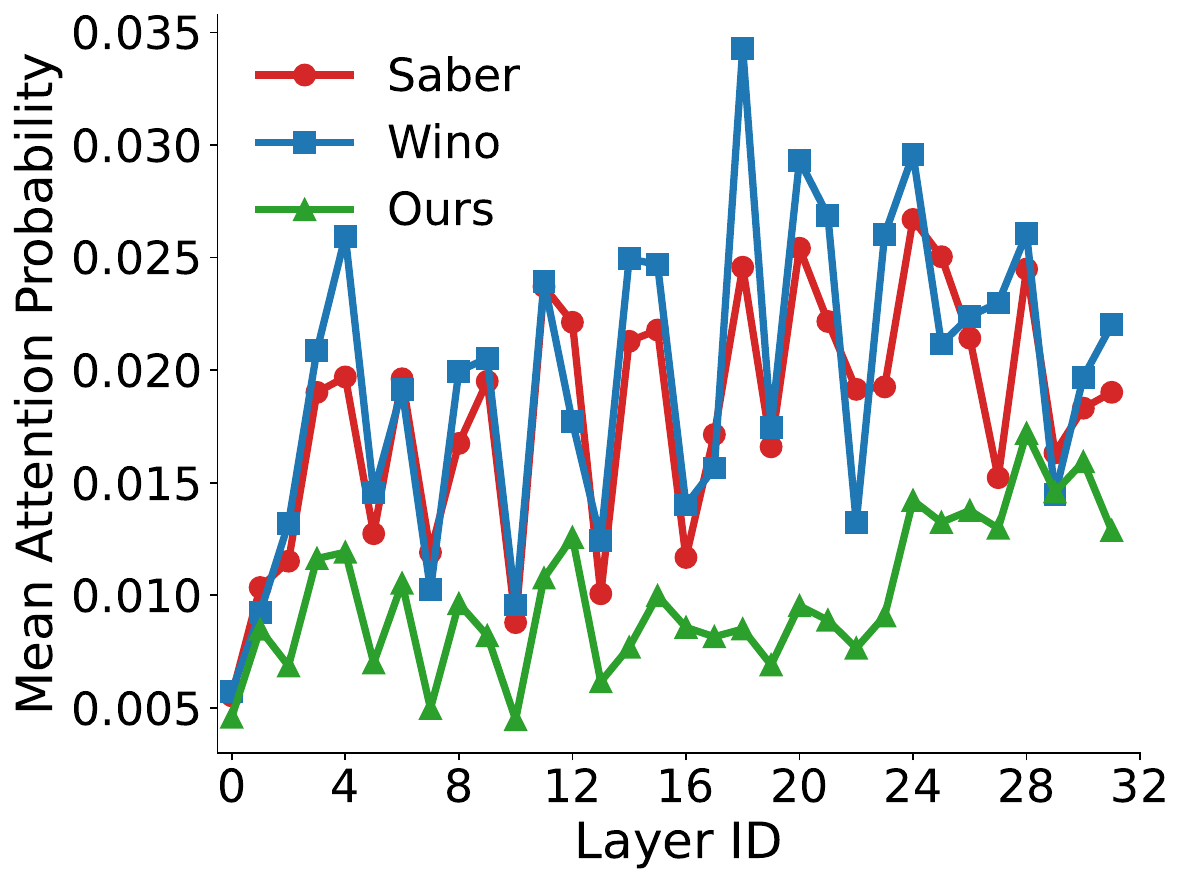}
        \caption{Pending-to-pending attention: Anchor-Perturbed Verification disrupts mutual reinforcement.}
        \label{fig:effect_3}
    \end{subfigure}
    \caption{\textit{Layer-wise attention redistribution induced by \sysname{}.} (a)--(b): \textbf{Anchor-Guided Generation} implicitly rectifies attention; mask tokens reallocate weight from pending neighbors to anchor tokens. (c)--(d): \textbf{Anchor-Perturbed Verification} induces a complementary effect; orthogonal probing increases the attention from pending tokens to anchors while dampening mutual reinforcement among pending tokens.}
    \label{fig:combined_analysis}
\end{figure*}

\section{Experiments}
\subsection{Experimental Setup}
\textbf{Implementation Details.}
We evaluate on two representative instruct dLLMs, LLaDA-Ins-8B~\citep{nie2025largelanguagediffusionmodels} and Dream-Ins-7B~\citep{ye2025dream}. Results on their base variants and on the RL-tuned LLaDA-1.5-8B~\citep{zhu2025llada15variancereducedpreference} are reported in Appendix~\ref{sec:additional_results}. For these comparisons, we adopt the Semi-AR sampling strategy from LLaDA with block size 32. Appendix~\ref{app:block_models} additionally evaluates native block diffusion models without replacing their decoders. Hyperparameter sensitivity is examined in Section~\ref{sec:ablation}, and additional configuration details are provided in Appendix~\ref{sec:experiments_details}.

\textbf{Datasets.}
We use four benchmarks covering two task families. GSM8K~\citep{cobbe2021trainingverifierssolvemath} and MATH500~\citep{lightman2024lets} target mathematical reasoning, while MBPP~\citep{austin2021programsynthesislargelanguage} and HumanEval~\citep{chen2021evaluatinglargelanguagemodels} target code generation. Detailed dataset descriptions are deferred to Appendix~\ref{sec:experiments_details}.

\textbf{Metrics.}
We report three metrics. For effectiveness, \textit{Accuracy} corresponds to final-answer accuracy on the math benchmarks and pass@1 on the code benchmarks. For efficiency, \textit{Steps} is the number of forward passes required to generate a fixed-length sequence, which provides a model-agnostic measure of decoding cost, and \textit{Speedup} is the wall-clock time reduction relative to the standard decoder, which we use as the primary efficiency metric.

\textbf{Baselines.}
We compare \sysname{} against the standard decoder and two state-of-the-art revocable-decoding methods, WINO~\citep{hong2025wide} and Saber~\citep{dong2025saber}. To ensure a fair comparison, both baselines are re-implemented from their official repositories within the LM-Eval~\citep{eval-harness} framework. We match the common decoding settings---including the backbone, output length, block size, and temperature---while retaining each method's native control parameters. Appendix~\ref{app:threshold_only} further isolates the anchor mechanisms from generic parallel drafting through a threshold-only control based on Fast-dLLM~\citep{wu2025fastdllmtrainingfreeaccelerationdiffusion}.

\subsection{Main Results}
\label{sec:main}
\textbf{Accuracy Performance.}
Table~\ref{tab:main_result} summarizes accuracy across all configurations. \sysname{} achieves the best average accuracy on every benchmark and backbone, with representative gains of $+4.9\%$ on LLaDA-Ins-8B for HumanEval at length 512 and $+6.4\%$ on Dream-Ins-7B for MATH500 at length 512. Existing baselines occasionally fall below the standard decoder, as in WINO on Dream-Ins-7B HumanEval, which reflects the cost of verifying against a context of mixed quality. The advantage of \sysname{} also amplifies with sequence length, growing from $+3.8\%$ to $+6.4\%$ on Dream-Ins-7B MATH500 as the length doubles, consistent with the ATC providing a stronger global skeleton over longer horizons. Appendix~\ref{app:multiseed} reports five-run results at temperature $0.2$ with 95\% confidence intervals; \sysname{} has the highest mean accuracy in all settings, although overlapping intervals mean that not every individual gain is statistically significant.

\textbf{Efficiency and Decoding Speed.}
\sysname{} achieves wall-clock speedups of $2.5\times$ to $7.2\times$ across all configurations. Anchor-Guided Generation produces higher-quality drafts that require fewer subsequent corrections, which directly translates accuracy gains into a smaller number of decoding steps. The largest speedup appears on Dream-Ins-7B for MBPP at length 512, where \sysname{} reaches $7.2\times$, compared to $5.2\times$ for WINO and $1.3\times$ for Saber. The efficiency advantage also scales with sequence length, growing from $4.1\times$ to $7.2\times$ on the same task.

\subsection{Empirical Evidence: Attention Redistribution}
To validate the implicit attention rectification claimed in Section~\ref{sec:mask_update}, the per-layer attention matrices are recorded and the mean attention probability between tokens of different roles is reported. Under Anchor-Guided Generation, mask tokens consistently allocate higher attention to anchors (Figure~\ref{fig:effect_2}) and lower attention to pending tokens (Figure~\ref{fig:effect_1}), with the effect most pronounced in deeper layers where semantic composition is believed to occur. An embedding-space injection alone is therefore enough to achieve attention rectification, without modifying any attention score and while preserving FlashAttention. A similar redistribution emerges under Anchor-Perturbed Verification (Figures~\ref{fig:effect_4}--\ref{fig:effect_3}): the orthogonal perturbation raises the attention from pending tokens to anchors while lowering the attention among pending tokens themselves, which complements the stability-probing rule of Section~\ref{sec:pending_update} by weakening local error reinforcement.

\subsection{Ablation Study}
\label{sec:ablation}

\begin{table*}[!ht]
\centering
\renewcommand{\arraystretch}{1.0}
\resizebox{1\linewidth}{!}{
\begin{tabular}{c l lrc lrc lrc lrc}
\toprule[1pt]
\multirow{2}{*}{\textbf{Length}} 
& \multirow{2}{*}{\textbf{Method}}
& \multicolumn{3}{c}{HumanEval}
& \multicolumn{3}{c}{MBPP}
& \multicolumn{3}{c}{GSM8K}
& \multicolumn{3}{c}{MATH500} \\

\cmidrule(lr){3-5} \cmidrule(lr){6-8} \cmidrule(lr){9-11} \cmidrule(lr){12-14}
& 
& Acc $\uparrow$ & Steps $\downarrow$ & Speed $\uparrow$
& Acc $\uparrow$ & Steps $\downarrow$ & Speed $\uparrow$
& Acc $\uparrow$ & Steps $\downarrow$ & Speed $\uparrow$
& Acc $\uparrow$ & Steps $\downarrow$ & Speed  $\uparrow$ \\

\hline
\multirow{4.5}{*}{256}
& baseline & \textcolor{DarkGrey}{38.7} & \textcolor{DarkGrey}{256.0} & \textcolor{DarkGrey}{1.0×} & \textcolor{DarkGrey}{36.8} & \textcolor{DarkGrey}{256.0} & \textcolor{DarkGrey}{1.0×} & \textcolor{DarkGrey}{77.4} & \textcolor{DarkGrey}{256.0} & \textcolor{DarkGrey}{1.0×} & \textcolor{DarkGrey}{33.8} & \textcolor{DarkGrey}{256.0} & \textcolor{DarkGrey}{1.0×} \\

& Ours & 41.5$_{\textcolor{LightGreen}{+2.8}}$ & 75.4 & 3.1× & 38.6$_{\textcolor{LightGreen}{+1.8}}$ & 88.3 & 2.5× & 79.5$_{\textcolor{LightGreen}{+2.1}}$ & 56.9 & 3.2× & 35.2$_{\textcolor{LightGreen}{+1.4}}$ & 74.8 & 2.9× \\

\cmidrule(lr){2-14}
& w/o Anchor-Guided Generation
& 39.0$_{\textcolor{red}{-2.5}}$ & 86.2 & 2.8× & 37.6$_{\textcolor{red}{-1.0}}$ & 102.5 & 2.3× & 78.4$_{\textcolor{red}{-1.1}}$ & 78.4 & 3.0× & 34.0$_{\textcolor{red}{-1.2}}$ & 89.6 & 2.7× \\

& w/o Anchor-Perturbed Verification
& 40.3$_{\textcolor{red}{-1.2}}$ & 68.9 & 3.5× & 38.0$_{\textcolor{red}{-0.6}}$ & 81.3 & 2.9× & 79.1$_{\textcolor{red}{-0.4}}$ & 67.1 & 3.5× & 33.6$_{\textcolor{red}{-1.6}}$ & 75.6 & 3.2× \\

\hline
\multirow{4.5}{*}{512}
& baseline & \textcolor{DarkGrey}{43.9} & \textcolor{DarkGrey}{512.0} & \textcolor{DarkGrey}{1.0×} & \textcolor{DarkGrey}{38.2} & \textcolor{DarkGrey}{512.0} & \textcolor{DarkGrey}{1.0×} & \textcolor{DarkGrey}{80.3} & \textcolor{DarkGrey}{512.0} & \textcolor{DarkGrey}{1.0×} & \textcolor{DarkGrey}{37.8} & \textcolor{DarkGrey}{512.0} & \textcolor{DarkGrey}{1.0×} \\

& Ours & 48.8$_{\textcolor{LightGreen}{+4.9}}$ & 127.8 & 3.6× & 39.4$_{\textcolor{LightGreen}{+1.2}}$ & 110.4 & 2.9× & 81.1$_{\textcolor{LightGreen}{+0.8}}$ & 88.5 & 4.3× & 38.8$_{\textcolor{LightGreen}{+1.0}}$ & 119.3 & 3.8×\\

\cmidrule(lr){2-14}
& w/o Anchor-Guided Generation
& 47.5$_{\textcolor{red}{-1.3}}$ & 148.2 & 3.3× & 38.4$_{\textcolor{red}{-1.0}}$ & 144.1 & 2.7× & 80.8$_{\textcolor{red}{-0.3}}$ & 104.5 & 4.1× & 38.6$_{\textcolor{red}{-0.2}}$ & 134.1 & 3.6×\\

& w/o Anchor-Perturbed Verification
& 46.9$_{\textcolor{red}{-1.9}}$ & 122.3 & 4.0× & 38.6$_{\textcolor{red}{-0.8}}$ & 121.7 & 3.2× & 80.4$_{\textcolor{red}{-0.7}}$ & 93.2 & 4.6× & 37.0$_{\textcolor{red}{-1.8}}$ & 117.8 & 4.1×\\

\bottomrule[1pt]
\end{tabular}
}
\caption{Module-level ablation on LLaDA-Ins-8B. ``w/o Anchor-Guided Generation'' disables the mask-side update ($\gamma^i_t=0$ on the mask role); ``w/o Anchor-Perturbed Verification'' disables the pending-side update ($\beta=0$ on the pending role). \textcolor{red}{Red} and \textcolor{LightGreen}{green} subscripts show the accuracy change relative to full \sysname{}. Design-level ablations (magnitude form on mask, direction form on pending) are reported in Table~\ref{tab:perturbation_type}.}
\label{tab:ablationresult}
\end{table*}

\begin{figure*}[ht]
    \centering
    \begin{subfigure}[b]{0.24\textwidth}
        \centering
        \includegraphics[width=\textwidth]{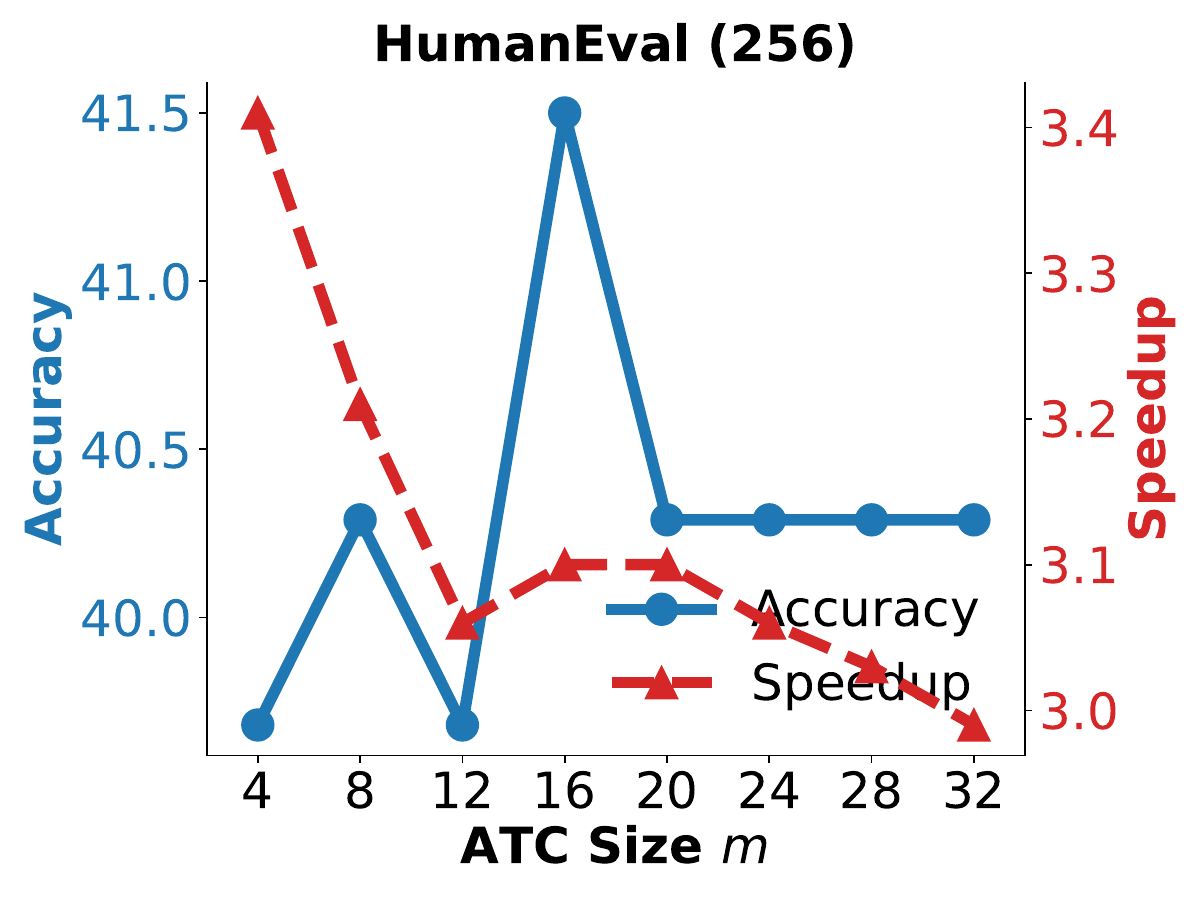}
        \label{fig:ATC_humaneval_256}
    \end{subfigure}
    \hfill
    \begin{subfigure}[b]{0.24\textwidth}
        \centering
        \includegraphics[width=\textwidth]{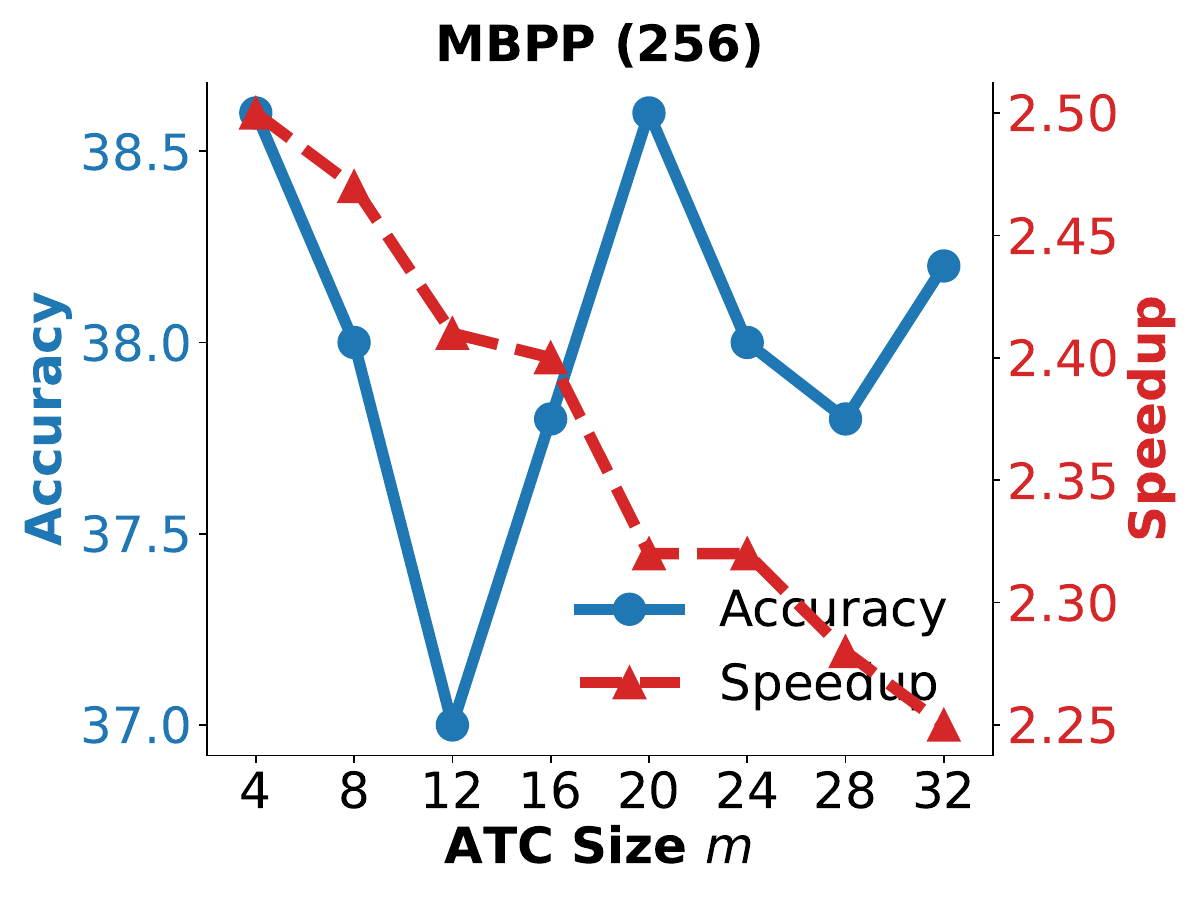}
        \label{fig:ATC_mbpp_256}
    \end{subfigure}
    \hfill
    \begin{subfigure}[b]{0.24\textwidth}
        \centering
        \includegraphics[width=\textwidth]{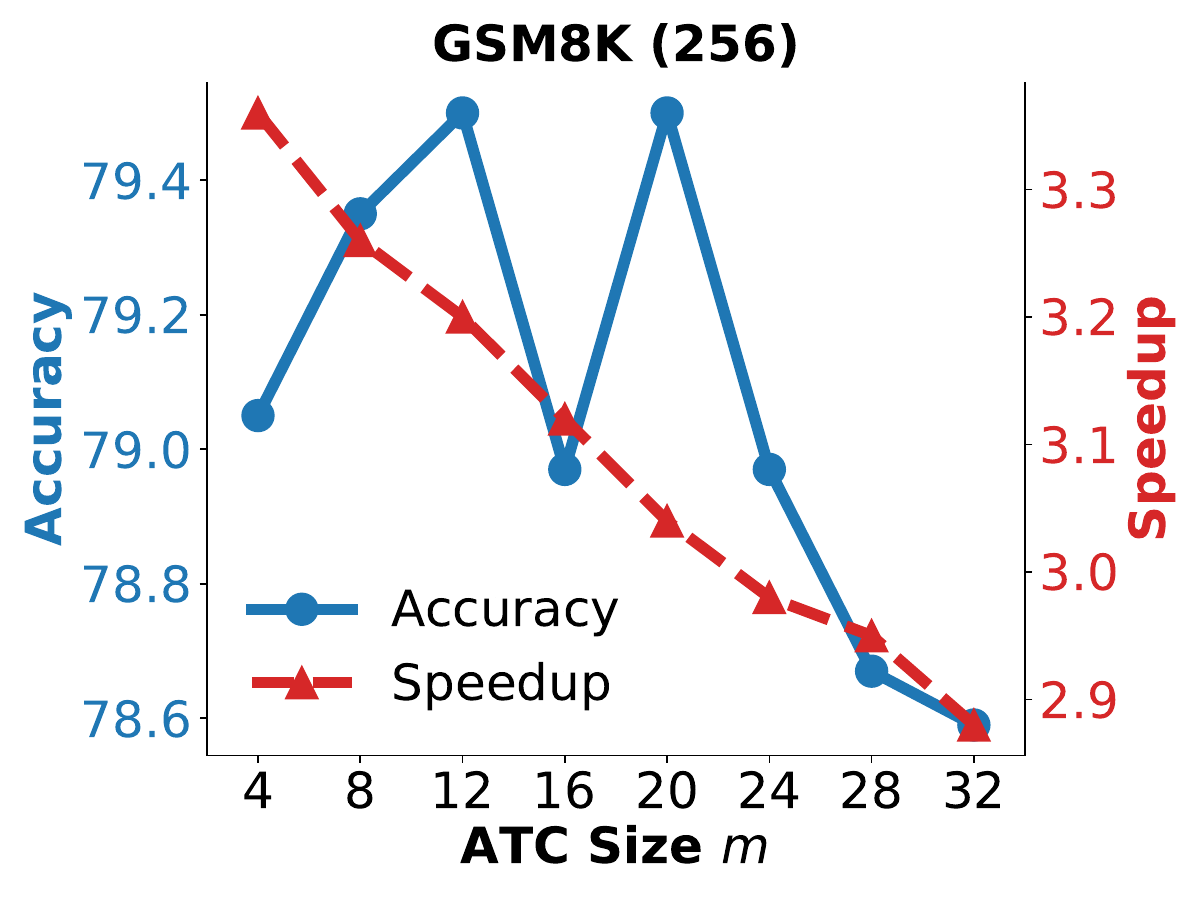}
        \label{fig:ATC_gsm8k_256}
    \end{subfigure}
    \hfill
    \begin{subfigure}[b]{0.24\textwidth}
        \centering
        \includegraphics[width=\textwidth]{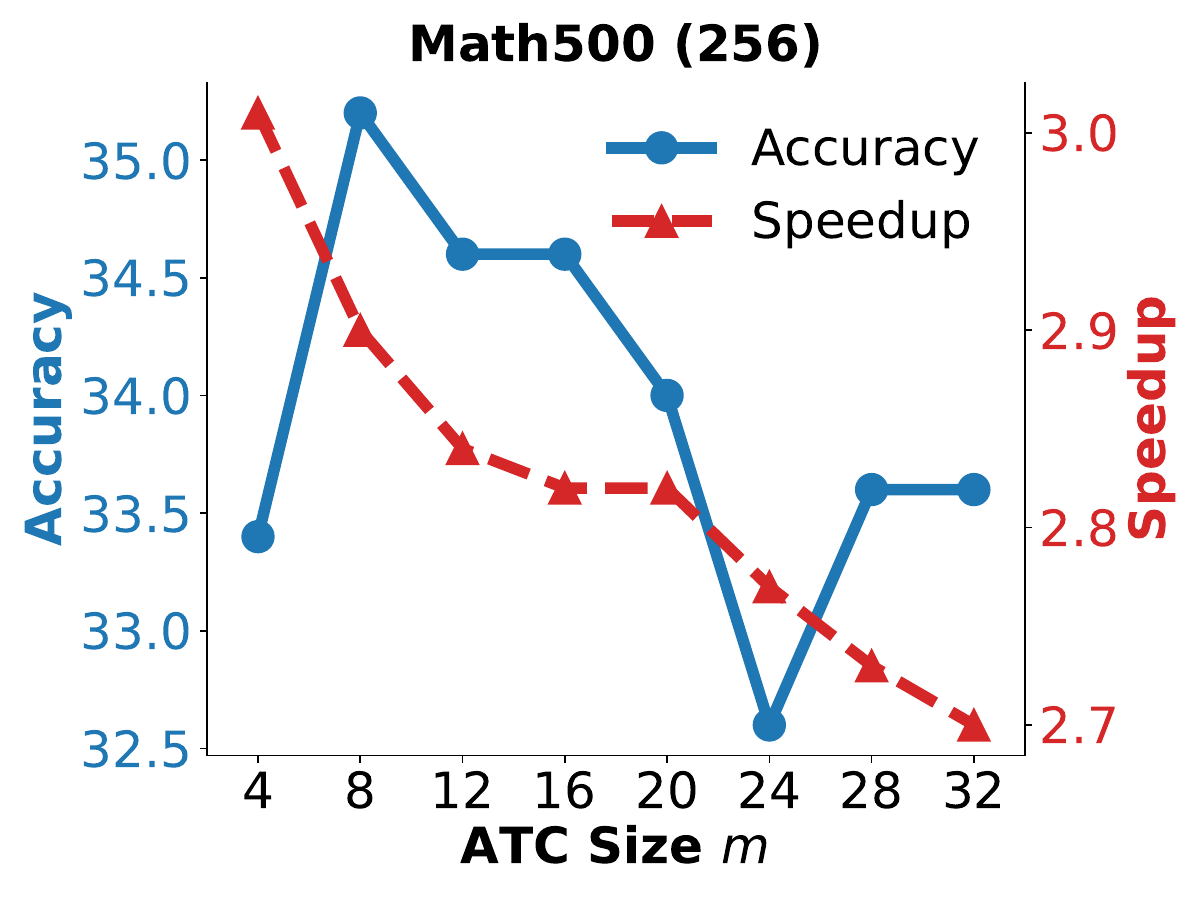}
        \label{fig:ATC_math500_256}
    \end{subfigure}

    \begin{subfigure}[b]{0.24\textwidth}
        \centering
        \includegraphics[width=\textwidth]{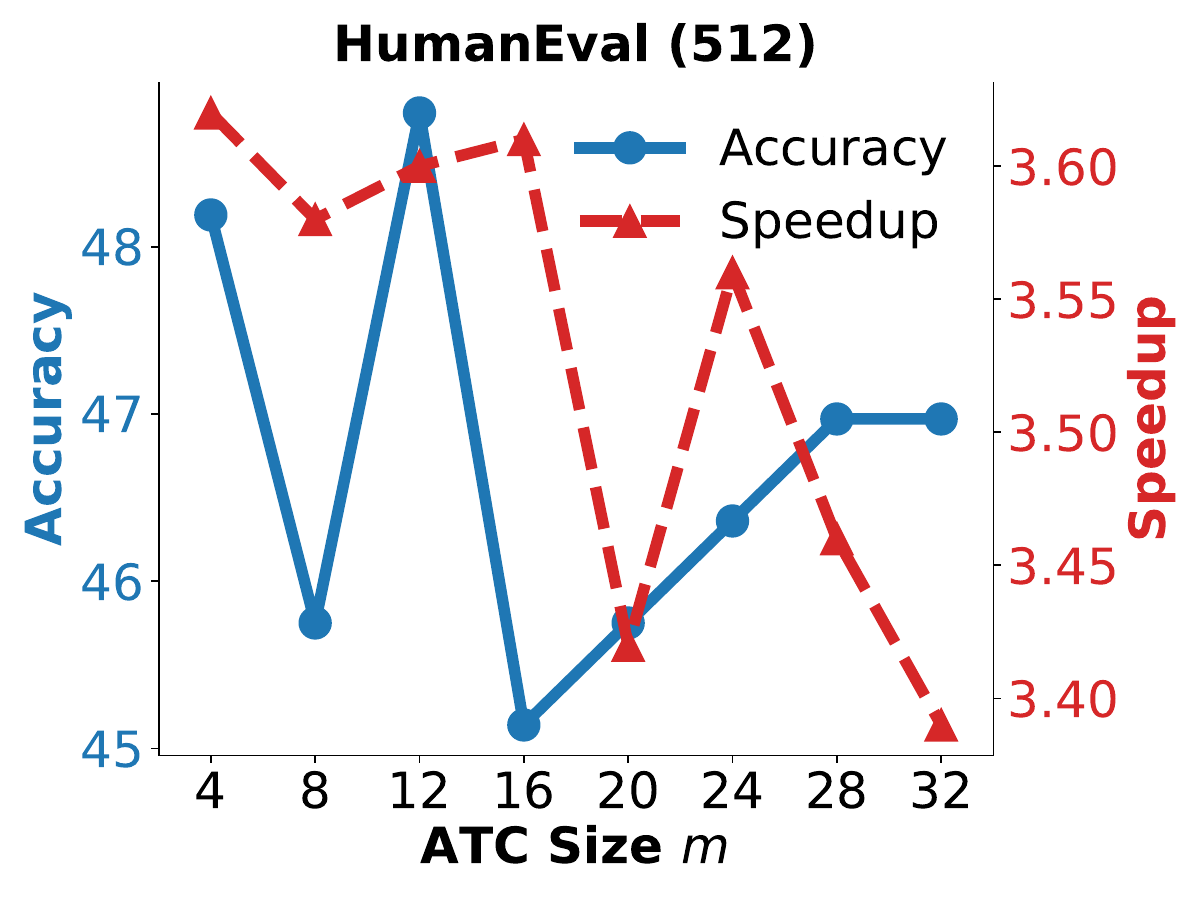}
        \label{fig:ATC_humaneval_512}
    \end{subfigure}
    \hfill
    \begin{subfigure}[b]{0.24\textwidth}
        \centering
        \includegraphics[width=\textwidth]{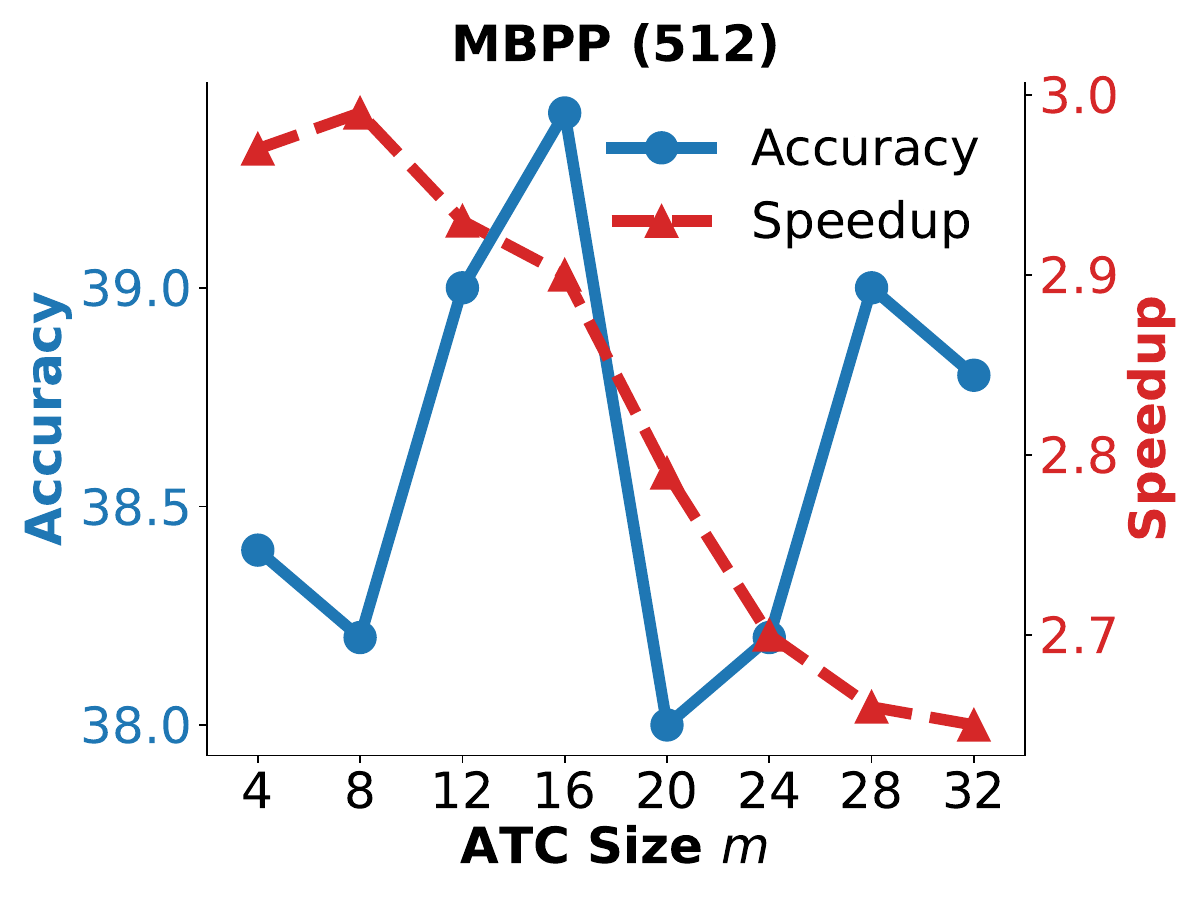}
        \label{fig:ATC_mbpp_512}
    \end{subfigure}
    \hfill
    \begin{subfigure}[b]{0.24\textwidth}
        \centering
        \includegraphics[width=\textwidth]{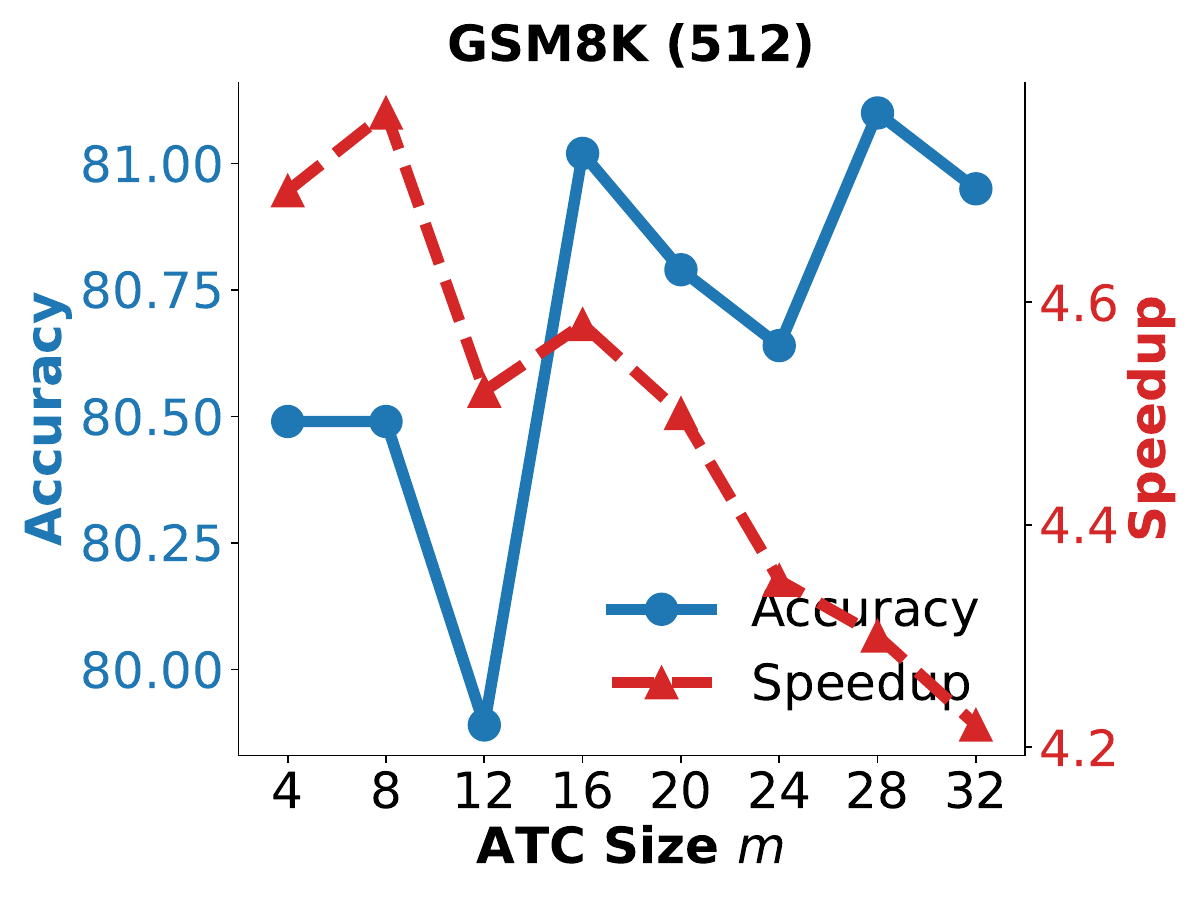}
        \label{fig:ATC_gsm8k_512}
    \end{subfigure}
    \hfill
    \begin{subfigure}[b]{0.24\textwidth}
        \centering
        \includegraphics[width=\textwidth]{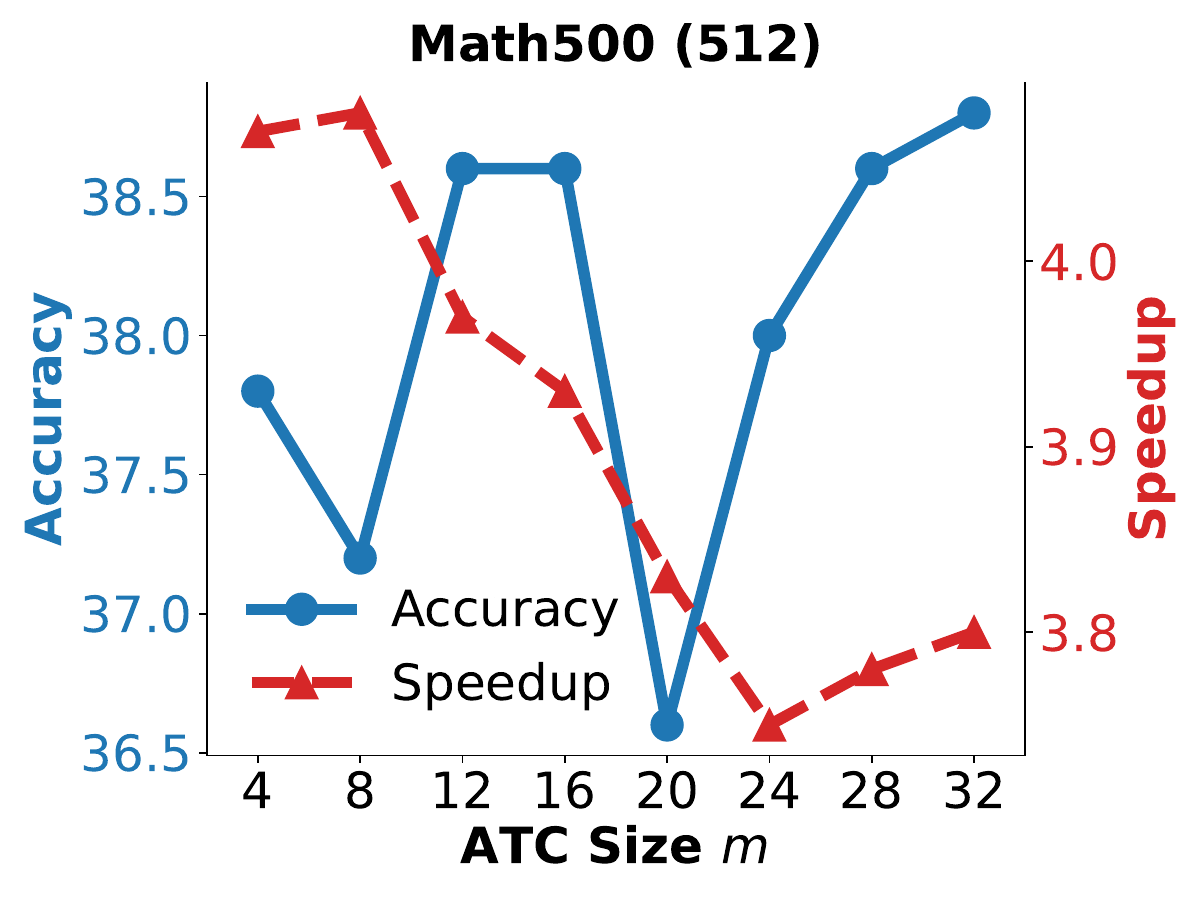}
        \label{fig:ATC_math500_512}
    \end{subfigure}
    \caption{Ablation on the ATC size $m\in[4,32]$ for LLaDA-Ins-8B across sequence lengths and benchmarks; block size 32 throughout.}
    \label{fig:ab_m}
\end{figure*}

\textbf{Module ablation.}
Table~\ref{tab:ablationresult} reports a module-level ablation. Removing Anchor-Guided Generation causes the largest accuracy loss on code generation, with HumanEval dropping by $2.5\%$ at length 256 and $1.3\%$ at length 512, while decoding steps grow correspondingly. Without an anchor-derived prior, the forward pass absorbs error from pending neighbors, which yields lower-quality drafts. Removing Anchor-Perturbed Verification consistently hurts MATH500, with drops of $1.6\%$ at length 256 and $1.8\%$ at length 512. The speedup correspondingly increases in this variant (e.g., from $3.6\times$ to $4.0\times$ on HumanEval at length 512), reflecting that verification introduces additional refinement steps.

\textbf{Effect of ATC size $m$.}
Figure~\ref{fig:ab_m} presents the sensitivity analysis of ATC capacity $m\in[4,32]$. Accuracy follows an inverted-U pattern: a cache that is too small cannot represent the sequence backbone, while one that is too large retains stale anchors that drift away from the current decoding frontier. The location of the peak is mildly task-dependent, but $m$ in $[12,20]$ works well across the four benchmarks. For speedup, larger $m$ values lead to lower throughput due to the increased computational overhead of anchor-derived operations.

\begin{table}[h!]
\centering
\renewcommand{\arraystretch}{1.05}
\setlength{\tabcolsep}{6pt}
\resizebox{\columnwidth}{!}{%
\begin{tabular}{lcccc}
\toprule
\textbf{Dataset} & $k$ & \textbf{Acc (\%) $\uparrow$} & \textbf{Anchors/sample} & \textbf{Contam.\,(\%) $\downarrow$} \\
\midrule
\multirow{4}{*}{GSM8K}
  & 1 & 79.7          & 307.6 & 6.68 \\
  & 2 & \textbf{81.1} & 101.5 & 2.56 \\
  & 3 & 81.0          & 87.1  & 1.79 \\
  & 4 & 78.5          & 35.3  & \textbf{1.74} \\
\midrule
\multirow{4}{*}{HumanEval}
  & 1 & 43.9          & 338.4 & 43.53 \\
  & 2 & \textbf{48.8} & 134.4 & 11.28 \\
  & 3 & 48.5          & 107.9 & 9.45  \\
  & 4 & 45.1          & 50.9  & \textbf{8.87}  \\
\bottomrule
\end{tabular}%
}
\caption{Effect of the consistency window $k$ for \sysname{} on LLaDA-Ins-8B (sequence length 512). \textit{Anchors/sample} is the average number of tokens promoted to the ATC per sample; \textit{Contam.\,(\%)} is the fraction of anchor tokens that fall inside error spans of the generated output, annotated token-level by Claude-Opus-4.6 against the reference solution (lower is better).}
\label{tab:k_ablation}
\end{table}

\textbf{Effect of consistency window $k$.}
Table~\ref{tab:k_ablation} sweeps $k\in\{1,2,3,4\}$ and additionally reports the fraction of anchor tokens that fall inside Claude-Opus-4.6-annotated error spans (\textit{Contam.}). As $k$ grows, contamination decays sharply on both benchmarks, which is the empirical signature of the $q^k$ bound in Proposition~\ref{prop:false_anchor}: each additional consistency step multiplies the false-anchor probability by $q\ll 1$. The number of admitted anchors shrinks even faster, since the same filter also rejects briefly fluctuating tokens. The two effects compose into a non-monotonic accuracy curve, and $k{=}2$ sits at the balance point between purity and density, which motivates it as the default.

\textbf{Effect of block size.}
Table~\ref{tab:blocksize_AB} reports the effect of the Semi-AR block size. A block size of 32 yields the best accuracy across all configurations. Doubling to 64 causes a mild degradation, while 128 leads to pronounced drops, with GSM8K falling to $72.4\%$ at length 512, since anchor tokens become too sparsely distributed to steer the current block. Speedup follows the same trend on GSM8K, whereas HumanEval throughput stays stable.

\begin{table}[!htbp]
\centering
\renewcommand{\arraystretch}{0.9}
\resizebox{1\linewidth}{!}{
\begin{tabular}{c l *{4}{c}}
\toprule[1pt]
\multirow{2}{*}{\textbf{Length}}
& \multirow{2}{*}{\textbf{Method}}
& \multicolumn{2}{c}{HumanEval}
& \multicolumn{2}{c}{GSM8K} \\
\cmidrule(lr){3-4} \cmidrule(lr){5-6}
&   & Acc $\uparrow$ & Speed $\uparrow$ & Acc $\uparrow$ & Speed $\uparrow$ \\
\hline
\multirow{4}{*}{256}
& baseline
& \textcolor{DarkGrey}{38.7} & \textcolor{DarkGrey}{1.0×}
& \textcolor{DarkGrey}{77.4} & \textcolor{DarkGrey}{1.0×} \\
& Ours (BS = 32) & \textbf{41.5} & \textbf{3.1×} & \textbf{79.5} & 3.2× \\
& Ours (BS = 64) & 40.9 & 3.0× & 78.3 & \textbf{3.3×} \\
& Ours (BS = 128) & 38.7 & 3.0× & 70.8 & 2.5× \\
\hline
\multirow{4}{*}{512}
& baseline
& \textcolor{DarkGrey}{43.9} & \textcolor{DarkGrey}{1.0×}
& \textcolor{DarkGrey}{80.3} & \textcolor{DarkGrey}{1.0×} \\
& Ours (BS = 32) & \textbf{48.8} & 3.6× & \textbf{81.1} & \textbf{4.3×} \\
& Ours (BS = 64) & 46.3 & 3.5× & 80.0 & 4.2× \\
& Ours (BS = 128) & 45.7 & \textbf{3.7×} & 72.4 & 2.3× \\
\bottomrule[1pt]
\end{tabular}
}
\caption{Effect of the Semi-AR block size on LLaDA-Ins-8B. Block size 32 attains the best accuracy across configurations.}
\label{tab:blocksize_AB}
\end{table}

\textbf{Sensitivity of $\alpha$ and $\beta$.}
Figure~\ref{fig:ab_alphabeta} sweeps the two knobs of the embedding-space updates: the mask-side guidance strength $\alpha$ and the pending-side perturbation strength $\beta$ defined in Section~\ref{sec:mask_update} and Section~\ref{sec:pending_update}. Both follow an inverted-U pattern with peaks in $\alpha\in[0.05,0.3]$ and $\beta\in[0.1,0.6]$. The two differ in their failure mode at large values: extreme $\alpha$ catastrophically overwrites the mask representation, whereas extreme $\beta$ degrades only gradually, since committed embeddings of pending tokens partially resist the orthogonal perturbation.

\begin{figure}[!htbp]
    \centering
    \begin{subfigure}[b]{0.23\textwidth}
        \centering
        \includegraphics[width=\textwidth]{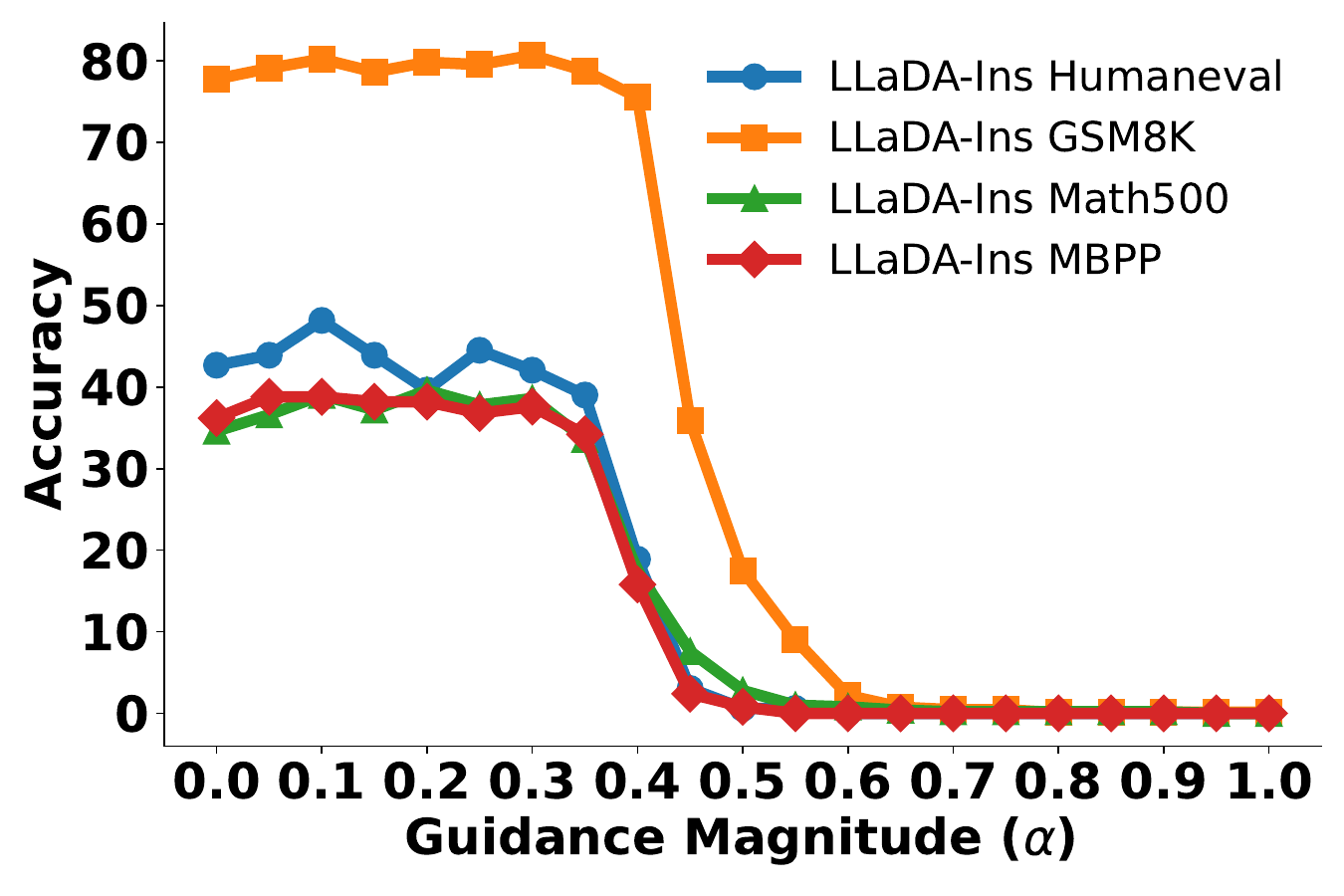}
        \caption{Guidance Strength $\alpha$}
        \label{fig:alpha}
    \end{subfigure}
    \hfill
    \begin{subfigure}[b]{0.23\textwidth}
        \centering
        \includegraphics[width=\textwidth]{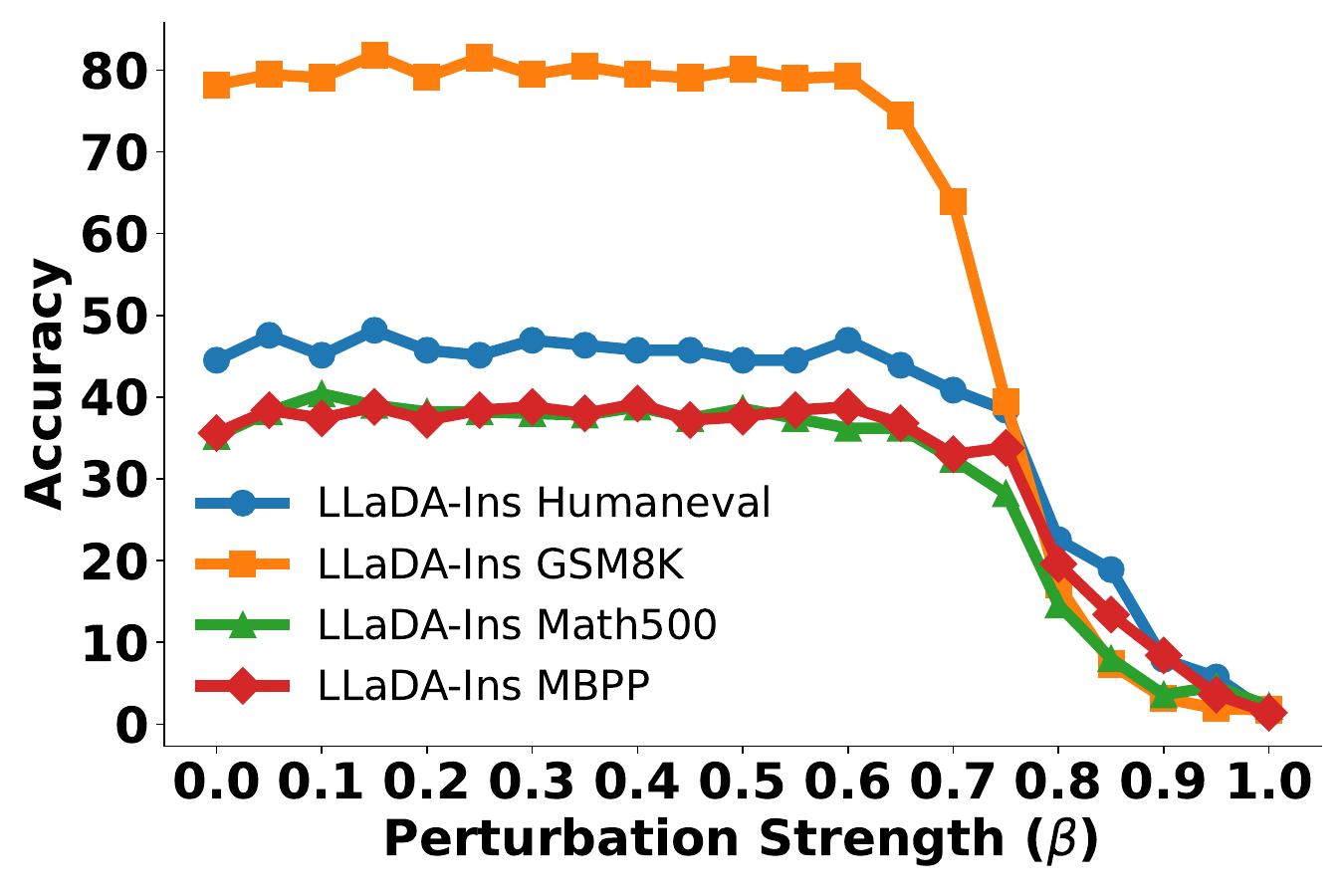}
        \caption{Perturbation Strength $\beta$}
        \label{fig:beta}
    \end{subfigure}
    \caption{Sensitivity to the two embedding-space-update knobs $\alpha$ (mask side) and $\beta$ (pending side) on LLaDA-Ins-8B at sequence length 512.}
    \label{fig:ab_alphabeta}
\end{figure}

\textbf{Design choices in the embedding-space updates.}
Table~\ref{tab:perturbation_type} dissects the two design degrees of freedom in Anchor-Guided Generation and Anchor-Perturbed Verification of Section~\ref{sec:mask_update} and Section~\ref{sec:pending_update}: the \textit{magnitude form} on the mask side and the \textit{direction form} on the pending side. For the magnitude form, replacing the entropy-modulated $\gamma^{i}_t = \alpha\bar{E}^{i}_t$ with a constant $\gamma^{i}_t = \alpha$ produces a consistent but moderate drop, which confirms that per-token entropy modulation contributes beyond the global $\alpha$. For the direction form, removing the probe loses up to $4.3\%$, while Gaussian noise and random embeddings lose up to $7.4\%$ and $6.6\%$, since they corrupt valid tokens by discarding the anchor signal. Even the parallel probe, which preserves the anchor signal but pushes along $\mathbf{e}^{i}_t$, fails to expose fragile predictions because it merely reinforces the token's existing direction. Orthogonality is therefore the load-bearing design choice.

\begin{table}[!htbp]
\centering
\renewcommand{\arraystretch}{0.9}
\resizebox{1\linewidth}{!}{
\begin{tabular}{c l llll}
\toprule[1pt]
\multirow{2}{*}{\textbf{Length}}
& \multirow{2}{*}{\textbf{Variant}}
& \multicolumn{2}{c}{HumanEval}
& \multicolumn{2}{c}{GSM8K} \\
\cmidrule(lr){3-4} \cmidrule(lr){5-6}
&  & Acc $\uparrow$ & Speed $\uparrow$ & Acc $\uparrow$ & Speed $\uparrow$ \\
\hline
\multirow{6}{*}{256}
& \textbf{Ours (full)}
& \textbf{41.5} & \textbf{3.1×} & \textbf{79.5} & \textbf{3.2×} \\
\cmidrule(lr){2-6}
& \multicolumn{5}{l}{\emph{Anchor-Guided Generation: magnitude form}} \\
& \quad Constant $\alpha$
& 40.3$_{\textcolor{red}{-1.2}}$  & 3.0× & 78.9$_{\textcolor{red}{-0.6}}$ & 3.1× \\
\cmidrule(lr){2-6}
& \multicolumn{5}{l}{\emph{Anchor-Perturbed Verification: direction form}} \\
& \quad Parallel probe
& 36.0$_{\textcolor{red}{-5.5}}$ & 2.8× & 77.1$_{\textcolor{red}{-2.4}}$ & 2.7× \\
& \quad Gaussian noise
& 38.4$_{\textcolor{red}{-3.1}}$ & 3.0× & 72.6$_{\textcolor{red}{-6.9}}$ & 2.6× \\
& \quad Random embedding
& 34.9$_{\textcolor{red}{-6.6}}$ & 2.7× & 77.2$_{\textcolor{red}{-2.3}}$ & 2.8× \\
& \quad No perturbation
& 37.2$_{\textcolor{red}{-4.3}}$ & 2.9× & 77.9$_{\textcolor{red}{-1.6}}$ & 3.0× \\
\hline
\multirow{6}{*}{512}
& \textbf{Ours (full)}
& \textbf{48.8} & \textbf{3.6×} & \textbf{81.1} & \textbf{4.3×} \\
\cmidrule(lr){2-6}
& \multicolumn{5}{l}{\emph{Anchor-Guided Generation: magnitude form}} \\
& \quad Constant $\alpha$
& 46.4$_{\textcolor{red}{-2.4}}$ & 3.4× & 80.3$_{\textcolor{red}{-0.8}}$ & 4.3× \\
\cmidrule(lr){2-6}
& \multicolumn{5}{l}{\emph{Anchor-Perturbed Verification: direction form}} \\
& \quad Parallel probe
& 42.6$_{\textcolor{red}{-6.2}}$ & 3.2× & 75.9$_{\textcolor{red}{-5.2}}$ & 3.9× \\
& \quad Gaussian noise
& 41.4$_{\textcolor{red}{-7.4}}$ & 3.0× & 75.1$_{\textcolor{red}{-6.0}}$ & 3.8× \\
& \quad Random embedding
& 44.5$_{\textcolor{red}{-4.3}}$ & 3.5× & 77.9$_{\textcolor{red}{-3.2}}$ & 4.1× \\
& \quad No perturbation
& 45.7$_{\textcolor{red}{-3.1}}$ & 3.3× & 78.1$_{\textcolor{red}{-3.0}}$ & 4.1× \\
\bottomrule[1pt]
\end{tabular}
}
\caption{Design-level ablation on the embedding-space updates with LLaDA-Ins-8B. \textit{Top block (Anchor-Guided Generation: magnitude form)}: Constant $\alpha$ replaces $\gamma^{i}_t=\alpha\bar{E}^{i}_t$ with $\gamma^{i}_t=\alpha$. \textit{Bottom block (Anchor-Perturbed Verification: direction form)}: Parallel probe projects $\bar{c}_t$ onto $\mathbf{e}^{i}_t$ rather than its orthogonal complement; Gaussian noise is isotropic noise; Random embedding is a random token embedding; No perturbation disables the perturbation.}
\label{tab:perturbation_type}
\end{table}

\section{Conclusion}

We present \sysname{}, a training-free revocable decoding framework for dLLMs that operates in the embedding space. By promoting temporally consistent tokens into a dynamic Anchor Token Cache, \sysname{} addresses error propagation and local error reinforcement through two complementary updates: Anchor-Guided Generation injects an entropy-weighted anchor signal into mask embeddings, and Anchor-Perturbed Verification probes committed tokens with an orthogonal anchor-derived signal to remask those sustained by spurious local consensus. Experiments on math and coding benchmarks show consistent accuracy gains across Dream and LLaDA. As a decoding-time procedure that does not alter the underlying model, \sysname{} inherits the risks of the host dLLM without amplifying them.





\section*{Limitations}
\label{sec:limitations}


Our evaluation of \sysname{} spans instruction-tuned, base, and RL-tuned variants of two diffusion LLM families (Dream and LLaDA) at the 7B--8B scale, covering closed-form math and code benchmarks as well as open-ended question answering and summarization. Several gaps nevertheless remain. The parameter scale we study is capped at 8B, so whether the same anchor dynamics carry over to substantially larger dLLMs, or to multi-turn dialogue and chat-style instruction-following settings, is left to future work. 
The decoding configuration is also not free of brittleness; accuracy degrades sharply at very large Semi-AR block sizes (GSM8K falls to $72.4\%$ at $\text{BS}{=}128$ on LLaDA-Ins-8B at sequence length 512), where anchors become too sparsely distributed within a block to supervise it.
The per-step overhead is modest ($7.6$\,ms, a $9.8\%$ increase on LLaDA-Ins-8B) and is amortized by the reduced step count, but the relative wall-clock benefit shrinks on very short sequences where only a handful of decoding steps are needed.
On the theoretical side, Proposition~\ref{prop:false_anchor} relies on idealized assumptions (i.i.d.\ noise across decoding steps and a positive semantic margin between the ground-truth token and its competitors) that need not hold in practice; in complex contexts an incorrect token may stay top-1 for several steps and slip into the ATC. We therefore treat the exponential bound as a qualitative insight rather than a guarantee on real dLLMs, and rely on the empirical false-anchor decay in Table~\ref{tab:k_ablation} as the actual evidence that the consistency filter suppresses incorrect anchors as $k$ grows.

\section*{Acknowledgments}
This work was supported in part by the UK Engineering and Physical Sciences Research Council (EPSRC) through a Turing AI Fellowship (grant no. EP/V020579/1, EP/V020579/2), KCL’s Impact Acceleration Account (grant no. EP/X525571/1), and in part by the National Natural Science Foundation of China (No. 625B2163, 62576120). A PhD studentship from the Chinese Scholarship Council funds Yizhen Yao. The authors also acknowledge the use of the King’s Computational Research, Engineering, and Technology Environment (CREATE) at King’s College London. 

\bibliography{sample-base}

\appendix
\appendix

\section{Detailed Experimental Setup}
\label{sec:experiments_details}

\subsection{Hyperparameters and Implementation}
For base models (Dream-Base-7B and LLaDA-Base-8B), we use the standard full-sequence diffusion sampling strategy and follow few-shot evaluation settings for each benchmark: zero-shot for HumanEval, 3-shot for MBPP, 4-shot for MATH500, and 8-shot for GSM8K. Note that the RL model in our experiments refers to LLaDA-1.5-8B, which is post-trained with reinforcement learning on top of LLaDA-Ins-8B. We report additional results on these three models (Dream-Base-7B, LLaDA-Base-8B, and LLaDA-1.5-8B) in Appendix~\ref{sec:additional_results}. We set the temperature to 0 for all experiments to ensure reproducibility. The unmasking threshold $\tau$ is selected from $\{0.6, 0.7, 0.8\}$ depending on the model and dataset. Since instruct models are trained with supervised fine-tuning (SFT), they exhibit more concentrated token distributions; accordingly, we use a slightly higher base threshold for instruct models to avoid premature decoding collapse. All experiments are conducted on 8 NVIDIA A100 80GB GPUs.

\textbf{Hyperparameter Details for \sysname.}
For our method, the ATC size $m$ is selected from $\{12, 16, 20\}$ depending on the sequence length and task complexity. The guidance strength $\alpha$ ranges over $[0.05, 0.2]$, with higher values providing stronger guidance at the cost of reduced exploration. The perturbation strength $\beta$ is selected from $[0.1, 0.4]$ for deployment, the sub-interval of the swept range $[0.1, 0.6]$ (Section~\ref{sec:ablation}) used in our main results, controlling the probe intensity for pending token verification. During the first $k-1$ decoding steps where the consistency window is incomplete, we seed the ATC with the single lowest-entropy threshold-passed token at each step to avoid an empty cache.

\textbf{Reporting protocol.} All reported numbers are from a single deterministic run per configuration with temperature $0$, so each cell is a point estimate rather than a mean over seeds.

\textbf{Baseline tuning for base and RL-tuned models.} Since the official WINO and Saber implementations do not specify hyperparameters for these backbones, we tune their native control parameters under the same evaluation framework used for \sysname{}. The standard-decoding rows use the same early-stopping protocol as the other standard-decoder results.

\begin{table*}[!t]
\centering
\scriptsize
\setlength{\tabcolsep}{3pt}
\renewcommand{\arraystretch}{0.95}
\begin{tabular}{lcl ccc ccc ccc ccc}
\toprule[1pt]
\multirow{2}{*}{\textbf{Model}} & \multirow{2}{*}{\textbf{Len}} & \multirow{2}{*}{\textbf{Method}} & \multicolumn{3}{c}{\textbf{HumanEval}} & \multicolumn{3}{c}{\textbf{MBPP}} & \multicolumn{3}{c}{\textbf{GSM8K}} & \multicolumn{3}{c}{\textbf{MATH500}} \\
\cmidrule(lr){4-6} \cmidrule(lr){7-9} \cmidrule(lr){10-12} \cmidrule(lr){13-15}
& & & Acc & Steps & Speed & Acc & Steps & Speed & Acc & Steps & Speed & Acc & Steps & Speed \\
\hline
\multirow{8}{*}{\textbf{Dream-Base-7B}} & \multirow{4}{*}{256} & Baseline & \textcolor{DarkGrey}{50.6} & \textcolor{DarkGrey}{238.4} & \textcolor{DarkGrey}{1.0×} & \textcolor{DarkGrey}{54.6} & \textcolor{DarkGrey}{239.1} & \textcolor{DarkGrey}{1.0×} & \textcolor{DarkGrey}{75.2} & \textcolor{DarkGrey}{237.5} & \textcolor{DarkGrey}{1.0×} & \textcolor{DarkGrey}{36.9} & \textcolor{DarkGrey}{240.2} & \textcolor{DarkGrey}{1.0×} \\
& & WINO & 48.8 & 125.4 & 1.5× & 55.8 & 87.7 & 2.7× & 75.7 & 137.8 & \textbf{1.9×} & 37.8 & 94.1 & 2.3× \\
& & Saber & 51.2 & 130.8 & 1.3× & 55.2 & 96.9 & 1.2× & 74.5 & 153.6 & 1.4× & 38.0 & 107.3 & 1.6× \\
& & \sysname{} & \textbf{54.3} & 111.4 & \textbf{2.0×} & \textbf{57.2} & 73.8 & \textbf{3.1×} & \textbf{76.3} & 116.2 & \textbf{1.9×} & \textbf{39.8} & 68.1 & \textbf{3.4×} \\
\cmidrule(lr){2-15}
& \multirow{4}{*}{512} & Baseline & \textcolor{DarkGrey}{54.3} & \textcolor{DarkGrey}{494.2} & \textcolor{DarkGrey}{1.0×} & \textcolor{DarkGrey}{55.8} & \textcolor{DarkGrey}{495.7} & \textcolor{DarkGrey}{1.0×} & \textcolor{DarkGrey}{75.5} & \textcolor{DarkGrey}{493.3} & \textcolor{DarkGrey}{1.0×} & \textcolor{DarkGrey}{37.4} & \textcolor{DarkGrey}{496.1} & \textcolor{DarkGrey}{1.0×} \\
& & WINO & 52.5 & 196.1 & 1.9× & 54.2 & 117.5 & 2.7× & 76.1 & 139.9 & 2.8× & 37.2 & 134.6 & 2.5× \\
& & Saber & 51.1 & 202.5 & 1.4× & 56.2 & 124.3 & 2.3× & 76.9 & 143.9 & 2.1× & 36.6 & 146.8 & 2.8× \\
& & \sysname{} & \textbf{54.9} & 195.1 & \textbf{2.4×} & \textbf{56.8} & 113.6 & \textbf{3.9×} & \textbf{77.8} & 138.4 & \textbf{3.3×} & \textbf{40.2} & 93.9 & \textbf{4.6×} \\
\midrule
\multirow{8}{*}{\textbf{LLaDA-Base-8B}} & \multirow{4}{*}{256} & Baseline & \textcolor{DarkGrey}{32.3} & \textcolor{DarkGrey}{240.5} & \textcolor{DarkGrey}{1.0×} & \textcolor{DarkGrey}{39.6} & \textcolor{DarkGrey}{238.8} & \textcolor{DarkGrey}{1.0×} & \textcolor{DarkGrey}{69.1} & \textcolor{DarkGrey}{237.2} & \textcolor{DarkGrey}{1.0×} & \textcolor{DarkGrey}{30.0} & \textcolor{DarkGrey}{236.9} & \textcolor{DarkGrey}{1.0×} \\
& & WINO & 34.1 & 204.6 & 1.2× & 40.6 & 121.5 & 1.2× & 70.5 & 134.9 & 1.1× & \textbf{31.2} & 113.8 & 1.3× \\
& & Saber & 32.3 & 246.2 & 1.5× & 40.0 & 185.5 & 1.7× & 69.7 & 160.5 & 1.7× & 30.8 & 150.5 & \textbf{2.1×} \\
& & \sysname{} & \textbf{34.8} & 124.7 & \textbf{1.8×} & \textbf{41.4} & 115.8 & \textbf{1.9×} & \textbf{71.3} & 125.3 & \textbf{1.8×} & \textbf{31.2} & 105.1 & \textbf{2.1×} \\
\cmidrule(lr){2-15}
& \multirow{4}{*}{512} & Baseline & \textcolor{DarkGrey}{32.9} & \textcolor{DarkGrey}{495.1} & \textcolor{DarkGrey}{1.0×} & \textcolor{DarkGrey}{39.8} & \textcolor{DarkGrey}{493.7} & \textcolor{DarkGrey}{1.0×} & \textcolor{DarkGrey}{70.8} & \textcolor{DarkGrey}{496.5} & \textcolor{DarkGrey}{1.0×} & \textcolor{DarkGrey}{30.8} & \textcolor{DarkGrey}{492.8} & \textcolor{DarkGrey}{1.0×} \\
& & WINO & 33.5 & 267.4 & 1.5× & 41.4 & 252.5 & 1.3× & 70.2 & 179.4 & 1.4× & 32.0 & 201.8 & 1.6× \\
& & Saber & 32.3 & 258.9 & \textbf{2.7×} & 41.2 & 363.4 & 1.8× & 70.3 & 282.6 & 2.6× & 30.2 & 296.4 & \textbf{2.5×} \\
& & \sysname{} & \textbf{34.8} & 179.2 & 2.6× & \textbf{42.0} & 180.4 & \textbf{2.6×} & \textbf{71.8} & 158.7 & \textbf{2.9×} & \textbf{32.2} & 194.5 & 2.4× \\
\midrule
\multirow{8}{*}{\textbf{LLaDA-1.5-8B}} & \multirow{4}{*}{256} & Baseline & \textcolor{DarkGrey}{38.4} & \textcolor{DarkGrey}{239.7} & \textcolor{DarkGrey}{1.0×} & \textcolor{DarkGrey}{38.6} & \textcolor{DarkGrey}{237.9} & \textcolor{DarkGrey}{1.0×} & \textcolor{DarkGrey}{79.2} & \textcolor{DarkGrey}{236.4} & \textcolor{DarkGrey}{1.0×} & \textcolor{DarkGrey}{33.4} & \textcolor{DarkGrey}{238.1} & \textcolor{DarkGrey}{1.0×} \\
& & WINO & 41.5 & 86.3 & 2.6× & 39.6 & 88.9 & 2.6× & 80.3 & 56.0 & 4.7× & 31.8 & 60.2 & 3.9× \\
& & Saber & 40.8 & 117.3 & 2.2× & 39.4 & 108.1 & 3.2× & 80.8 & 72.4 & 4.2× & 32.8 & 82.1 & 4.5× \\
& & \sysname{} & \textbf{42.1} & 72.3 & \textbf{3.2×} & \textbf{39.8} & 63.5 & \textbf{3.6×} & \textbf{81.1} & 44.1 & \textbf{5.1×} & \textbf{33.6} & 46.2 & \textbf{4.8×} \\
\cmidrule(lr){2-15}
& \multirow{4}{*}{512} & Baseline & \textcolor{DarkGrey}{45.1} & \textcolor{DarkGrey}{493.5} & \textcolor{DarkGrey}{1.0×} & \textcolor{DarkGrey}{37.6} & \textcolor{DarkGrey}{496.8} & \textcolor{DarkGrey}{1.0×} & \textcolor{DarkGrey}{82.9} & \textcolor{DarkGrey}{494.2} & \textcolor{DarkGrey}{1.0×} & \textcolor{DarkGrey}{38.6} & \textcolor{DarkGrey}{495.3} & \textcolor{DarkGrey}{1.0×} \\
& & WINO & 46.9 & 139.2 & 2.3× & 36.8 & 107.4 & 2.8× & 83.7 & 78.7 & 4.7× & 38.8 & 87.6 & \textbf{5.2×} \\
& & Saber & 47.0 & 166.9 & 3.6× & 37.2 & 151.8 & 3.7× & 82.6 & 130.7 & 5.2× & 38.6 & 178.0 & 4.5× \\
& & \sysname{} & \textbf{47.6} & 107.4 & \textbf{4.1×} & \textbf{38.4} & 106.8 & \textbf{4.1×} & \textbf{83.9} & 71.6 & \textbf{6.4×} & \textbf{40.2} & 89.1 & \textbf{5.2×} \\
\bottomrule[1pt]
\end{tabular}
\caption{Competitive decoding results on base and RL-tuned models. Each benchmark reports accuracy (Acc), decoding steps (Steps), and wall-clock speedup (Speed) over the standard decoder. Dream-Base-7B and LLaDA-Base-8B are base models, while LLaDA-1.5-8B is RL-tuned. Baseline results are shown in \textcolor{DarkGrey}{grey}; \textbf{bold} marks the best or tied-best accuracy and speedup in each setting.}
\label{tab:diff_models}
\end{table*}

\begin{table*}[!t]
\centering
\small
\setlength{\tabcolsep}{4pt}
\renewcommand{\arraystretch}{0.95}
\begin{tabular}{l l cc cccc}
\toprule[1pt]
\multirow{2}{*}{\textbf{Model}} & \multirow{2}{*}{\textbf{Method}}
& \multicolumn{2}{c}{\textbf{NQ-Open}}
& \multicolumn{4}{c}{\textbf{CNN/DailyMail}} \\
\cmidrule(lr){3-4} \cmidrule(lr){5-8}
& & EM & Speed & BERTScore & R-1 & R-L & Speed \\
\midrule
\multirow{4}{*}{\textbf{Dream-Ins-7B}}
 & Baseline   & \textcolor{DarkGrey}{15.8} & \textcolor{DarkGrey}{1.0$\times$} & \textcolor{DarkGrey}{85.5} & \textcolor{DarkGrey}{22.6} & \textcolor{DarkGrey}{15.6} & \textcolor{DarkGrey}{1.0$\times$} \\
 & WINO       & 16.1 & 1.7$\times$ & 85.8 & 23.2 & 15.8 & 1.4$\times$ \\
 & Saber      & 16.4 & 1.4$\times$ & 85.6 & 22.9 & 16.0 & 1.9$\times$ \\
 & \sysname{} & \textbf{17.5}$_{\textcolor{LightGreen}{+1.7}}$ & \textbf{2.1$\times$} & \textbf{86.0} & \textbf{23.6} & \textbf{16.5} & \textbf{2.3$\times$} \\
\midrule
\multirow{4}{*}{\textbf{LLaDA-Ins-8B}}
 & Baseline   & \textcolor{DarkGrey}{12.3} & \textcolor{DarkGrey}{1.0$\times$} & \textcolor{DarkGrey}{82.2} & \textcolor{DarkGrey}{19.8} & \textcolor{DarkGrey}{13.7} & \textcolor{DarkGrey}{1.0$\times$} \\
 & WINO       & 13.4 & 2.0$\times$ & 82.3 & 20.3 & 13.8 & 2.2$\times$ \\
 & Saber      & 13.2 & 1.6$\times$ & 82.5 & 19.9 & 14.0 & 1.8$\times$ \\
 & \sysname{} & \textbf{14.6}$_{\textcolor{LightGreen}{+2.3}}$ & \textbf{2.3$\times$} & \textbf{82.8} & \textbf{20.5} & \textbf{14.3} & \textbf{2.5$\times$} \\
\bottomrule[1pt]
\end{tabular}
\caption{Open-ended generation results at sequence length $512$. NQ-Open is scored by exact match (EM); CNN/DailyMail is scored by BERTScore, ROUGE-1, and ROUGE-L. Speed is wall-clock speedup over the standard decoder for each task.}
\label{tab:open_ended}
\end{table*}

\subsection{Datasets}
We evaluate on four benchmarks covering mathematical reasoning and code generation:
\begin{itemize}
    \item \textbf{GSM8K}~\citep{cobbe2021trainingverifierssolvemath}: A dataset of grade-school-level multi-step math word problems requiring arithmetic reasoning.
    \item \textbf{MATH500}~\citep{lightman2024lets}: A uniformly sampled 500-problem subset of the MATH test set that is representative across its seven competition-mathematics subjects.
    \item \textbf{MBPP}~\citep{austin2021programsynthesislargelanguage}: A benchmark of basic Python function synthesis tasks designed to test fundamental programming capabilities.
    \item \textbf{HumanEval}~\citep{chen2021evaluatinglargelanguagemodels}: A set of 164 manually curated programming problems with unit-test-based evaluation.
\end{itemize}

\section{Additional Experimental Results}
\label{sec:additional_results}

\subsection{Multi-Seed Evaluation}
\label{app:multiseed}

Decoding at temperature $0$ is deterministic, so the main results do not admit seed-based variance estimates. We therefore rerun all methods at temperature $0.2$ with five random seeds. Table~\ref{tab:multiseed} reports the mean accuracy and the 95\% Student's $t$-confidence interval across the five runs, together with mean decoding steps and mean wall-clock speedup. \sysname{} achieves the highest mean accuracy in all $16$ model--length--benchmark settings. However, several confidence intervals overlap, so these results do not establish that every individual improvement is statistically significant.

\begin{table*}[!t]
\centering
\scriptsize
\setlength{\tabcolsep}{2.2pt}
\renewcommand{\arraystretch}{0.92}
\resizebox{\textwidth}{!}{%
\begin{tabular}{lcl ccc ccc ccc ccc}
\toprule[1pt]
\multirow{2}{*}{\textbf{Model}} & \multirow{2}{*}{\textbf{Len}} & \multirow{2}{*}{\textbf{Method}} & \multicolumn{3}{c}{\textbf{HumanEval}} & \multicolumn{3}{c}{\textbf{MBPP}} & \multicolumn{3}{c}{\textbf{GSM8K}} & \multicolumn{3}{c}{\textbf{MATH500}} \\
\cmidrule(lr){4-6} \cmidrule(lr){7-9} \cmidrule(lr){10-12} \cmidrule(lr){13-15}
& & & Acc $\pm$ CI & Steps & Speed & Acc $\pm$ CI & Steps & Speed & Acc $\pm$ CI & Steps & Speed & Acc $\pm$ CI & Steps & Speed \\
\midrule
\multirow{8}{*}{\textbf{LLaDA-Ins-8B}} & \multirow{4}{*}{256} & Baseline & \textcolor{DarkGrey}{$38.4\pm2.06$} & \textcolor{DarkGrey}{256.0} & \textcolor{DarkGrey}{1.0$\times$} & \textcolor{DarkGrey}{$36.5\pm0.88$} & \textcolor{DarkGrey}{256.0} & \textcolor{DarkGrey}{1.0$\times$} & \textcolor{DarkGrey}{$77.3\pm0.86$} & \textcolor{DarkGrey}{256.0} & \textcolor{DarkGrey}{1.0$\times$} & \textcolor{DarkGrey}{$34.0\pm0.71$} & \textcolor{DarkGrey}{256.0} & \textcolor{DarkGrey}{1.0$\times$} \\
& & WINO & $37.3\pm1.65$ & 75.5 & 2.3$\times$ & $36.1\pm1.52$ & 91.1 & 1.9$\times$ & $77.3\pm1.09$ & 53.9 & 2.9$\times$ & $34.4\pm0.79$ & 79.4 & 2.0$\times$ \\
& & Saber & $39.2\pm2.32$ & 102.9 & 2.4$\times$ & $37.3\pm0.98$ & 112.2 & 2.2$\times$ & $77.8\pm0.66$ & 119.9 & 1.9$\times$ & $35.0\pm0.37$ & 136.0 & 1.8$\times$ \\
& & \sysname{} & $\mathbf{41.3\pm2.29}$ & 75.6 & 3.1$\times$ & $\mathbf{38.2\pm0.82}$ & 87.3 & 2.5$\times$ & $\mathbf{79.5\pm0.96}$ & 57.1 & 3.2$\times$ & $\mathbf{35.4\pm0.95}$ & 75.1 & 2.9$\times$ \\
\cmidrule(lr){2-15}
& \multirow{4}{*}{512} & Baseline & \textcolor{DarkGrey}{$44.3\pm2.34$} & \textcolor{DarkGrey}{512.0} & \textcolor{DarkGrey}{1.0$\times$} & \textcolor{DarkGrey}{$38.4\pm0.36$} & \textcolor{DarkGrey}{512.0} & \textcolor{DarkGrey}{1.0$\times$} & \textcolor{DarkGrey}{$80.4\pm0.50$} & \textcolor{DarkGrey}{512.0} & \textcolor{DarkGrey}{1.0$\times$} & \textcolor{DarkGrey}{$37.7\pm0.75$} & \textcolor{DarkGrey}{512.0} & \textcolor{DarkGrey}{1.0$\times$} \\
& & WINO & $44.2\pm3.32$ & 130.6 & 2.9$\times$ & $37.5\pm0.40$ & 134.4 & 2.4$\times$ & $78.7\pm0.63$ & 83.3 & 3.6$\times$ & $35.7\pm0.83$ & 118.3 & 2.7$\times$ \\
& & Saber & $44.9\pm2.38$ & 195.7 & 2.3$\times$ & $37.9\pm0.23$ & 188.2 & 2.6$\times$ & $79.7\pm0.56$ & 144.4 & 3.1$\times$ & $37.4\pm0.96$ & 191.0 & 2.5$\times$ \\
& & \sysname{} & $\mathbf{49.0\pm2.60}$ & 125.4 & 3.7$\times$ & $\mathbf{39.4\pm0.62}$ & 110.0 & 2.9$\times$ & $\mathbf{81.1\pm0.56}$ & 89.1 & 4.3$\times$ & $\mathbf{39.2\pm0.65}$ & 116.5 & 3.9$\times$ \\
\midrule
\multirow{8}{*}{\textbf{Dream-Ins-7B}} & \multirow{4}{*}{256} & Baseline & \textcolor{DarkGrey}{$55.0\pm1.09$} & \textcolor{DarkGrey}{256.0} & \textcolor{DarkGrey}{1.0$\times$} & \textcolor{DarkGrey}{$57.6\pm0.43$} & \textcolor{DarkGrey}{256.0} & \textcolor{DarkGrey}{1.0$\times$} & \textcolor{DarkGrey}{$80.3\pm0.54$} & \textcolor{DarkGrey}{256.0} & \textcolor{DarkGrey}{1.0$\times$} & \textcolor{DarkGrey}{$37.9\pm1.62$} & \textcolor{DarkGrey}{256.0} & \textcolor{DarkGrey}{1.0$\times$} \\
& & WINO & $51.5\pm1.02$ & 85.6 & 2.4$\times$ & $57.1\pm0.60$ & 65.1 & 3.6$\times$ & $78.5\pm0.46$ & 72.5 & 3.4$\times$ & $38.4\pm1.85$ & 128.2 & 1.6$\times$ \\
& & Saber & $52.5\pm1.83$ & 165.1 & 1.5$\times$ & $57.8\pm0.24$ & 191.0 & 1.3$\times$ & $79.8\pm0.30$ & 182.8 & 1.3$\times$ & $38.9\pm0.90$ & 180.7 & 1.4$\times$ \\
& & \sysname{} & $\mathbf{56.2\pm1.22}$ & 84.2 & 3.0$\times$ & $\mathbf{58.8\pm0.34}$ & 62.2 & 4.1$\times$ & $\mathbf{80.7\pm0.75}$ & 66.3 & 3.9$\times$ & $\mathbf{41.6\pm1.75}$ & 78.5 & 3.5$\times$ \\
\cmidrule(lr){2-15}
& \multirow{4}{*}{512} & Baseline & \textcolor{DarkGrey}{$56.2\pm1.22$} & \textcolor{DarkGrey}{512.0} & \textcolor{DarkGrey}{1.0$\times$} & \textcolor{DarkGrey}{$56.3\pm0.47$} & \textcolor{DarkGrey}{512.0} & \textcolor{DarkGrey}{1.0$\times$} & \textcolor{DarkGrey}{$79.8\pm0.60$} & \textcolor{DarkGrey}{512.0} & \textcolor{DarkGrey}{1.0$\times$} & \textcolor{DarkGrey}{$38.5\pm1.80$} & \textcolor{DarkGrey}{512.0} & \textcolor{DarkGrey}{1.0$\times$} \\
& & WINO & $51.1\pm1.13$ & 99.9 & 4.0$\times$ & $55.4\pm0.67$ & 85.3 & 5.3$\times$ & $79.1\pm0.51$ & 93.4 & 5.5$\times$ & $39.6\pm2.06$ & 141.9 & 2.8$\times$ \\
& & Saber & $53.2\pm2.04$ & 261.8 & 1.9$\times$ & $55.8\pm0.26$ & 385.2 & 1.3$\times$ & $79.4\pm0.33$ & 318.5 & 1.6$\times$ & $41.5\pm1.00$ & 254.5 & 2.0$\times$ \\
& & \sysname{} & $\mathbf{56.8\pm1.35}$ & 83.8 & 5.6$\times$ & $\mathbf{57.3\pm0.38}$ & 75.2 & 7.1$\times$ & $\mathbf{80.8\pm0.83}$ & 79.1 & 6.7$\times$ & $\mathbf{44.8\pm1.94}$ & 95.7 & 5.4$\times$ \\
\bottomrule[1pt]
\end{tabular}%
}
\caption{Five-run evaluation at temperature $0.2$. Accuracy is reported as the mean $\pm$ the 95\% Student's $t$-confidence-interval half-width across five random seeds; Steps and Speed are run means. Speed is wall-clock speedup over the standard decoder. Baseline results are shown in \textcolor{DarkGrey}{grey}, and \textbf{bold} marks the highest mean accuracy in each setting.}
\label{tab:multiseed}
\end{table*}

\subsection{Generalization Across Base and RL-Tuned dLLMs}

Table~\ref{tab:diff_models} evaluates \sysname{} against WINO and Saber on two base models and one RL-tuned model. \sysname{} achieves the best or tied-best accuracy in all $24$ model--length--benchmark settings and the best or tied-best speedup in $22$ of them. Within the LLaDA family, its speedup is consistently higher on the RL-tuned model than on the base model across both generation lengths and all four benchmarks. This trend suggests that stronger post-trained representations can yield more reliable anchor supervision, although isolating the effect of RL post-training from other model differences requires further study.

\subsection{Generalization to Native Block Diffusion Models}
\label{app:block_models}

We further evaluate \sysname{} on three native block diffusion backbones: SDAR-8B-Chat~\citep{cheng-etal-2026-sdar}, LLaDA2.0-mini~\citep{bie2025llada2}, and LLaDA2.1-mini~\citep{bie2026llada21}. \sysname{} is applied without replacing their native block decoders. For LLaDA2.1, we retain its joint Mask-to-Token (M2T) and Token-to-Token (T2T) decoder: stable filled or edited tokens can become anchors, while candidates that fail \sysname{} verification are returned to the native editable pool.

\begin{table*}[!tbp]
\centering
\small
\setlength{\tabcolsep}{4pt}
\renewcommand{\arraystretch}{0.95}
\begin{tabular}{l l cc cc cc cc}
\toprule[1pt]
\multirow{2}{*}{\textbf{Backbone}} & \multirow{2}{*}{\textbf{Decoder}}
& \multicolumn{2}{c}{\textbf{HumanEval}}
& \multicolumn{2}{c}{\textbf{MBPP}}
& \multicolumn{2}{c}{\textbf{GSM8K}}
& \multicolumn{2}{c}{\textbf{MATH500}} \\
\cmidrule(lr){3-4} \cmidrule(lr){5-6} \cmidrule(lr){7-8} \cmidrule(lr){9-10}
& & Acc & Speed & Acc & Speed & Acc & Speed & Acc & Speed \\
\midrule
\multirow{2}{*}{SDAR-8B-Chat}
& Native      & \textcolor{DarkGrey}{62.1} & \textcolor{DarkGrey}{1.0$\times$} & \textcolor{DarkGrey}{59.4} & \textcolor{DarkGrey}{1.0$\times$} & \textcolor{DarkGrey}{84.2} & \textcolor{DarkGrey}{1.0$\times$} & \textcolor{DarkGrey}{47.4} & \textcolor{DarkGrey}{1.0$\times$} \\
& +\sysname{} & \textbf{63.3} & \textbf{1.5$\times$} & \textbf{60.2} & \textbf{1.3$\times$} & \textbf{84.7} & \textbf{1.9$\times$} & \textbf{48.8} & \textbf{1.7$\times$} \\
\midrule
\multirow{2}{*}{LLaDA2.0-mini}
& Native      & \textcolor{DarkGrey}{54.3} & \textcolor{DarkGrey}{1.0$\times$} & \textcolor{DarkGrey}{54.6} & \textcolor{DarkGrey}{1.0$\times$} & \textcolor{DarkGrey}{86.5} & \textcolor{DarkGrey}{1.0$\times$} & \textcolor{DarkGrey}{44.6} & \textcolor{DarkGrey}{1.0$\times$} \\
& +\sysname{} & \textbf{56.7} & \textbf{1.7$\times$} & \textbf{55.2} & \textbf{1.9$\times$} & \textbf{87.2} & \textbf{2.1$\times$} & \textbf{45.2} & \textbf{2.0$\times$} \\
\midrule
\multirow{2}{*}{LLaDA2.1-mini}
& Native      & \textcolor{DarkGrey}{53.1} & \textcolor{DarkGrey}{1.0$\times$} & \textcolor{DarkGrey}{52.8} & \textcolor{DarkGrey}{1.0$\times$} & \textcolor{DarkGrey}{85.7} & \textcolor{DarkGrey}{1.0$\times$} & \textcolor{DarkGrey}{42.0} & \textcolor{DarkGrey}{1.0$\times$} \\
& +\sysname{} & \textbf{57.9} & \textbf{2.2$\times$} & \textbf{55.2} & \textbf{2.5$\times$} & \textbf{86.8} & \textbf{2.6$\times$} & \textbf{44.6} & \textbf{1.8$\times$} \\
\bottomrule[1pt]
\end{tabular}
\caption{Generalization to native block diffusion models at generation length $512$ and block size $32$. Accuracy is reported in percent and speed is wall-clock speedup over each backbone's native decoder. \sysname{} augments rather than replaces the native decoding process.}
\label{tab:block_models}
\end{table*}

Table~\ref{tab:block_models} reports results at generation length $512$ and block size $32$. We omit length $256$ because these models typically produce longer outputs and are frequently truncated under that budget. Moreover, the native generation lengths of LLaDA2.0 and LLaDA2.1 are $2048$ and $4096$, respectively; the fixed-length results here should therefore not be interpreted as their best attainable performance. Across all three backbones and four benchmarks, \sysname{} improves accuracy by $0.5$--$4.8$ points while providing $1.3\times$--$2.6\times$ speedups. The particularly large gains on LLaDA2.1 are consistent with the possibility that \sysname{} benefits from a decoder trained for T2T correction, although isolating this interaction requires further study.

\subsection{Threshold-Only Parallel Drafting Control}
\label{app:threshold_only}

To separate the gains of the anchor mechanisms from those of threshold-based parallel drafting, we construct a threshold-only baseline using the parallel-decoding component of Fast-dLLM~\citep{wu2025fastdllmtrainingfreeaccelerationdiffusion}. We evaluate this control, the standard decoder, and \sysname{} on Dream-Ins-7B with sequence length $512$, using the same temperature ($0$), block size ($32$), and generation length for all methods.

\begin{table*}[!tbp]
\centering
\small
\setlength{\tabcolsep}{4pt}
\renewcommand{\arraystretch}{0.95}
\begin{tabular}{l cc cc cc cc}
\toprule[1pt]
\multirow{2}{*}{\textbf{Method}}
& \multicolumn{2}{c}{\textbf{HumanEval}}
& \multicolumn{2}{c}{\textbf{MBPP}}
& \multicolumn{2}{c}{\textbf{GSM8K}}
& \multicolumn{2}{c}{\textbf{MATH500}} \\
\cmidrule(lr){2-3} \cmidrule(lr){4-5} \cmidrule(lr){6-7} \cmidrule(lr){8-9}
& Acc & Speed & Acc & Speed & Acc & Speed & Acc & Speed \\
\midrule
Standard decoder              & \textcolor{DarkGrey}{56.1} & \textcolor{DarkGrey}{1.0$\times$} & \textcolor{DarkGrey}{56.2} & \textcolor{DarkGrey}{1.0$\times$} & \textcolor{DarkGrey}{80.2} & \textcolor{DarkGrey}{1.0$\times$} & \textcolor{DarkGrey}{38.6} & \textcolor{DarkGrey}{1.0$\times$} \\
\sysname{}                    & \textbf{56.7} & 5.5$\times$ & \textbf{57.2} & 7.2$\times$ & \textbf{81.1} & \textbf{6.8$\times$} & \textbf{45.0} & \textbf{5.4$\times$} \\
\midrule
Threshold-only ($\tau{=}0.95$) & 52.4 & 2.7$\times$ & 57.0 & 4.6$\times$ & 80.4 & 4.5$\times$ & 39.4 & 2.4$\times$ \\
Threshold-only ($\tau{=}0.9$)  & 52.4 & 3.0$\times$ & 57.0 & 4.9$\times$ & 80.0 & 4.9$\times$ & 38.8 & 2.6$\times$ \\
Threshold-only ($\tau{=}0.8$)  & 50.0 & 3.4$\times$ & 56.2 & 5.7$\times$ & 79.6 & 5.2$\times$ & 37.2 & 3.0$\times$ \\
Threshold-only ($\tau{=}0.7$)  & 46.3 & 5.2$\times$ & 52.4 & 6.9$\times$ & 77.2 & 5.6$\times$ & 37.0 & 3.7$\times$ \\
Threshold-only ($\tau{=}0.6$)  & 42.7 & \textbf{5.7$\times$} & 48.8 & \textbf{7.7$\times$} & 74.8 & 6.5$\times$ & 32.6 & 5.1$\times$ \\
\bottomrule[1pt]
\end{tabular}
\caption{Controlled comparison with threshold-only parallel drafting on Dream-Ins-7B at sequence length $512$. Accuracy is reported in percent and speed is wall-clock speedup over the standard decoder. All methods use temperature $0$ and block size $32$.}
\label{tab:threshold_only}
\end{table*}

Table~\ref{tab:threshold_only} exposes the quality--speed trade-off of threshold-only decoding: lowering $\tau$ increases parallelism and speed but substantially reduces accuracy. For example, on HumanEval, reducing $\tau$ from $0.95$ to $0.6$ raises the speedup from $2.7\times$ to $5.7\times$ while lowering accuracy from $52.4$ to $42.7$. In contrast, \sysname{} preserves or improves upon the standard decoder on all four benchmarks while achieving speedups of $5.4\times$--$7.2\times$. Thus, its gains cannot be explained by the underlying threshold-based drafting rule alone; the anchor-guided generation and verification mechanisms are necessary to retain quality at high parallelism.


\subsection{Open-Ended Generation}
\label{app:open_ended}

The four benchmarks in Section~\ref{sec:main} all emit short, closed-form answers: numeric outputs for the math benchmarks and single-function code completions for the coding benchmarks. To test whether \sysname{} transfers to free-form generation, we evaluate two open-ended tasks at sequence length $512$: NQ-Open~\citep{kwiatkowski-etal-2019-natural,lee-etal-2019-latent} for open-domain question answering, scored by exact match (EM), and CNN/DailyMail~\citep{see-etal-2017-get} for abstractive summarization, scored by BERTScore, ROUGE-1 (R-1), and ROUGE-L (R-L).

\sysname{} improves both quality and speed on both tasks for both backbones (Table~\ref{tab:open_ended}). The largest accuracy gain is on Dream-Ins-7B for NQ-Open ($15.8 \to 17.5$ EM with a $2.1\times$ speedup), and BERTScore improvements are consistent across both models. These results indicate that the anchor mechanism transfers beyond closed-form benchmarks to settings where multiple valid outputs exist.

\subsection{Perturbation Diagnostics and Case Study}
\label{app:perturb_diag}

\begin{table}[H]
\centering
\small
\setlength{\tabcolsep}{4pt}
\renewcommand{\arraystretch}{0.95}
\begin{tabular}{l cc cc}
\toprule[1pt]
\multirow{2}{*}{\textbf{Method}}
& \multicolumn{2}{c}{\textbf{HumanEval}}
& \multicolumn{2}{c}{\textbf{GSM8K}} \\
\cmidrule(lr){2-3} \cmidrule(lr){4-5}
& Error (\%) $\downarrow$ & $R$ (\%) $\downarrow$ & Error (\%) $\downarrow$ & $R$ (\%) $\downarrow$ \\
\midrule
Baseline   & \textcolor{DarkGrey}{35.52} & \textcolor{DarkGrey}{59.73} & \textcolor{DarkGrey}{17.15} & \textcolor{DarkGrey}{48.35} \\
WINO       & 32.22 & 50.66 & 15.46 & 39.25 \\
\sysname{} & \textbf{28.17} & \textbf{16.28} & \textbf{11.25} & \textbf{11.83} \\
\bottomrule[1pt]
\end{tabular}
\caption{Perturbation diagnostics on LLaDA-Ins-8B (sequence length 512, block size 32). Errors are annotated token-level by Claude-Opus-4.6 against the reference solution, following the same protocol as Table~\ref{tab:k_ablation}. The error rate is the fraction of generated tokens marked as errors; the long-span ratio $R$ is the fraction of those erroneous tokens that belong to a contiguous run of at least $15$ such tokens.}
\label{tab:perturb_metrics}
\end{table}

We report two diagnostic metrics on the generated outputs of LLaDA-Ins-8B at sequence length $512$ with block size $32$, comparing the standard decoder (Baseline), WINO, and \sysname{}. For each generated sequence, we ask Claude-Opus-4.6 to align it against the reference solution and emit a token-level error mask using the protocol detailed in Appendix~\ref{app:claude_protocol}.

\sysname{} reduces both quantities on both benchmarks (Table~\ref{tab:perturb_metrics}), and the reduction in $R$ is consistently sharper than the reduction in error rate. On HumanEval, $R$ falls from $59.73$ to $16.28$ while the error rate falls from $35.52$ to $28.17$; on GSM8K, $R$ falls from $48.35$ to $11.83$ while the error rate falls from $17.15$ to $11.25$. This pattern is the empirical signature of Anchor-Perturbed Verification (Section~\ref{sec:pending_update}): by destabilizing mutually reinforcing token clusters, the probe converts long error runs into short scattered errors that are easier for downstream verification to catch.

\subsubsection{Claude Annotation Protocol}
\label{app:claude_protocol}

\textbf{Model and decoding settings.}
We use Claude-Opus-4.6 as the annotator. Claude is queried with its default decoding configuration, with no additional sampling parameters overridden. The outputs judged in the perturbation diagnostics are generated by LLaDA-Ins-8B with sequence length $512$, Semi-AR block size $32$, unmasking threshold $\tau=0.7$, and temperature $0$; \sysname{} uses the default consistency window $k=2$.

\textbf{Task-specific prompts.}
For HumanEval, we use the following prompt:
\begin{quote}
``You are an expert code evaluator. Given a problem specification, a reference solution, and a generated solution, identify contiguous token-level error spans in the generated solution. Mark spans that are semantically or functionally incorrect with respect to the specification or reference solution, including undefined identifiers, wrong operators or conditions, incorrect return expressions, incorrect control flow, syntax/runtime errors, or logic that would fail the intended tests. Do not mark stylistic differences, variable-name differences, or alternative implementations if they are functionally equivalent. Return JSON containing overall correctness, error spans with token indices, error type, rationale, and a token-level binary error mask.''
\end{quote}

For GSM8K, we use the following prompt:
\begin{quote}
``You are an expert mathematical reasoning evaluator. Given a problem, a reference solution, and a generated reasoning trace, identify contiguous token-level error spans in the generated output. Mark spans containing incorrect quantities, operations, equations, intermediate conclusions, reasoning dependencies, or final answers. Do not mark paraphrases or mathematically equivalent reasoning as errors. If an early mistake causes later steps to be wrong, mark both the original erroneous span and later spans that explicitly depend on the incorrect value or conclusion. Return JSON containing overall correctness, error spans with token indices, error type, rationale, and a token-level binary error mask.''
\end{quote}

\textbf{Annotation example.}
For LLaDA's output on GSM8K instance \#38, \sysname{} correctly interprets ``half as much the other two days'' as $1.5$ hours on each day, yielding $6$ total hours and a speed of $10$ mph. WINO instead treats the third run as half of the second day's $1.5$ hours, i.e., $0.75$ hours, yielding $5.25$ total hours and $11.33$ mph. Claude marks the initial causal error span (``another half of the second day ... 0.75 hours'') and the propagated error spans containing $0.75$, $5.25$, and $11.33$.

\textbf{Human validation.}
We randomly sample $100$ Claude-annotated outputs from each benchmark for manual evaluation by members of our group. Each annotation is classified as \textsc{Correct}, \textsc{Partially Correct}, or \textsc{Incorrect}. \textsc{Partially Correct} denotes cases in which the output contains a mixture of correct and incorrect span annotations. Table~\ref{tab:claude_validation} reports the resulting counts.

\begin{table}[t]
\centering
\small
\setlength{\tabcolsep}{5pt}
\begin{tabular}{lccc}
\toprule
\textbf{Benchmark} & \textbf{Correct} & \textbf{Partial} & \textbf{Incorrect} \\
\midrule
HumanEval & 97 & 3 & 0 \\
GSM8K & 98 & 2 & 0 \\
\bottomrule
\end{tabular}
\caption{Manual validation of $100$ Claude annotations per benchmark. Values are annotation counts.}
\label{tab:claude_validation}
\end{table}

\textbf{Case study.}
In HumanEval problem 92, the WINO decoder commits to \texttt{x==sum\_xy or y==sum\_xy}, where the identifier \texttt{sum\_xy} is never defined anywhere in the generated function. The expression is internally self-consistent, since both clauses share the same undefined symbol, and therefore evades local verification: any check that compares the two clauses against each other finds them mutually compatible. Under \sysname{}, the orthogonal probe destabilizes this cluster, the affected tokens are remasked, and the regenerated expression \texttt{(x==y+z) or (y==x+z)} closes the gap with no undefined identifiers.

\subsection{Per-Step Runtime Breakdown}
\label{app:runtime}

\begin{table}[!htbp]
\centering
\small
\setlength{\tabcolsep}{4pt}
\renewcommand{\arraystretch}{0.95}
\begin{subtable}[t]{\linewidth}
\centering
\begin{tabular}{l r}
\toprule[1pt]
\textbf{Component} & \textbf{Time (ms)} \\
\midrule
Baseline: Forward         & 74.3 \\
Baseline: Post-process    & 3.0 \\
\textbf{Baseline: Total}  & \textbf{77.3} \\
\midrule
\sysname{}: Forward                & 74.1 \\
\sysname{}: ATC management         & 1.3 \\
\sysname{}: Guided injection       & 3.4 \\
\sysname{}: Perturbed verification & 3.1 \\
\sysname{}: Post-process           & 3.0 \\
\textbf{\sysname{}: Total}         & \textbf{84.9} \\
\bottomrule[1pt]
\end{tabular}
\caption{Per-step component breakdown.}
\label{tab:runtime_components}
\end{subtable}

\begin{subtable}[t]{\linewidth}
\centering
\begin{tabular}{l c c}
\toprule[1pt]
\textbf{Metric} & \textbf{Baseline} & \textbf{\sysname{}} \\
\midrule
Average steps      & \textcolor{DarkGrey}{256} & 75.4 \\
Wall-clock (s)     & \textcolor{DarkGrey}{19.8}  & 6.4 \\
Speedup            & \textcolor{DarkGrey}{1.0$\times$} & \textbf{3.1$\times$} \\
\bottomrule[1pt]
\end{tabular}
\caption{End-to-end cost and speedup.}
\label{tab:runtime_endtoend}
\end{subtable}
\caption{Runtime breakdown on LLaDA-Ins-8B (sequence length 256, A100 80GB). Speedup is wall-clock over the standard decoder. \emph{Guided injection} and \emph{Perturbed verification} refer to Anchor-Guided Generation (Section~\ref{sec:mask_update}) and Anchor-Perturbed Verification (Section~\ref{sec:pending_update}), respectively.}
\label{tab:runtime}
\end{table}

We profile the per-step cost of \sysname{} on LLaDA-Ins-8B with sequence length 256 on a single NVIDIA A100 80GB GPU, averaging over the generated samples. Table~\ref{tab:runtime_components} breaks the cost down by component, and Table~\ref{tab:runtime_endtoend} reports the end-to-end consequence in steps and wall-clock time.

\sysname{} adds $7.6$ ms per step on average, a $9.8\%$ increase over the baseline's $77.3$ ms. The overhead is concentrated in the embedding-space update: $3.4$ ms for guided injection and $3.1$ ms for perturbed verification, both linear in the cache size $m$. Despite this per-step overhead, the average step count drops from $256$ to $75.4$, so wall-clock time falls from $19.8$ s to $6.4$ s, giving a $3.1\times$ end-to-end speedup.

\begin{table*}[!t]
\centering
\small
\setlength{\tabcolsep}{4pt}
\renewcommand{\arraystretch}{0.95}
\begin{tabular}{l l cc cc cc cc}
\toprule[1pt]
\multirow{2}{*}{\textbf{Dataset}} & \multirow{2}{*}{\textbf{Method}}
& \multicolumn{2}{c}{$\text{BS}{=}16$} & \multicolumn{2}{c}{$\text{BS}{=}32$} & \multicolumn{2}{c}{$\text{BS}{=}64$} & \multicolumn{2}{c}{$\text{BS}{=}128$} \\
\cmidrule(lr){3-4} \cmidrule(lr){5-6} \cmidrule(lr){7-8} \cmidrule(lr){9-10}
& & Acc & Speed & Acc & Speed & Acc & Speed & Acc & Speed \\
\midrule
\multirow{4}{*}{MATH500}
 & Baseline   & \textcolor{DarkGrey}{32.4} & \textcolor{DarkGrey}{1.0$\times$} & \textcolor{DarkGrey}{33.8} & \textcolor{DarkGrey}{1.0$\times$} & \textcolor{DarkGrey}{34.2} & \textcolor{DarkGrey}{1.0$\times$} & \textcolor{DarkGrey}{31.6} & \textcolor{DarkGrey}{1.0$\times$} \\
 & WINO       & 33.6 & 1.9$\times$ & 34.2 & 2.0$\times$ & 34.0 & 2.0$\times$ & 32.2 & 1.8$\times$ \\
 & Saber      & 34.0 & 1.7$\times$ & 34.8 & 1.8$\times$ & 34.4 & 1.8$\times$ & 32.4 & 1.7$\times$ \\
 & \sysname{} & \textbf{34.6} & \textbf{2.8$\times$} & \textbf{35.2} & \textbf{2.9$\times$} & \textbf{34.6} & \textbf{2.6$\times$} & \textbf{32.6} & \textbf{2.8$\times$} \\
\midrule
\multirow{4}{*}{MBPP}
 & Baseline   & \textcolor{DarkGrey}{37.8} & \textcolor{DarkGrey}{1.0$\times$} & \textcolor{DarkGrey}{36.8} & \textcolor{DarkGrey}{1.0$\times$} & \textcolor{DarkGrey}{36.8} & \textcolor{DarkGrey}{1.0$\times$} & \textcolor{DarkGrey}{34.2} & \textcolor{DarkGrey}{1.0$\times$} \\
 & WINO       & 36.6 & 1.8$\times$ & 36.2 & 1.9$\times$ & 37.4 & 2.0$\times$ & 36.4 & 1.8$\times$ \\
 & Saber      & 37.4 & 2.0$\times$ & 37.8 & 2.1$\times$ & 37.8 & 2.1$\times$ & 36.6 & 1.9$\times$ \\
 & \sysname{} & \textbf{38.2} & \textbf{2.5$\times$} & \textbf{38.6} & \textbf{2.5$\times$} & \textbf{38.0} & \textbf{2.3$\times$} & \textbf{36.8} & \textbf{2.6$\times$} \\
\bottomrule[1pt]
\end{tabular}
\caption{Block-size sweep on LLaDA-Ins-8B. Accuracy in percent; speed is wall-clock speedup over the standard decoder.}
\label{tab:block_size_sweep}
\end{table*}

\subsection{Sensitivity to Block Size and Drafting Strength}
\label{app:block_schedule}

The main comparisons in Section~\ref{sec:main} use the method-specific settings described in Appendix~\ref{sec:experiments_details}. The experiments below are controlled sensitivity analyses, rather than the configurations used throughout the main comparison. To assess whether the gains of \sysname{} carry over to other decoding configurations, we separately vary block size and drafting strength on MATH500 and MBPP with LLaDA-Ins-8B.

\textbf{Block size sweep.}
Table~\ref{tab:block_size_sweep} reports accuracy and wall-clock speedup for $\text{BS}\in\{16,32,64,128\}$. On MATH500, \sysname{} reaches $35.2$ accuracy at $\text{BS}{=}32$ with a $2.9\times$ speedup, exceeding WINO ($34.2$, $2.0\times$) and Saber ($34.8$, $1.8\times$); the ranking holds at every block size we tested. On MBPP, \sysname{} likewise dominates both baselines across all block sizes. Accuracy degrades for all methods at $\text{BS}{=}128$, consistent with the observation in Section~\ref{sec:main} that overly wide blocks weaken anchor supervision.

\begin{table*}[!t]
\centering
\small
\setlength{\tabcolsep}{4pt}
\renewcommand{\arraystretch}{0.95}
\begin{tabular}{l l cc cc cc cc}
\toprule[1pt]
\multicolumn{10}{c}{\emph{Drafting threshold $\tau_d$ (WINO, \sysname{}; $\tau_v$ fixed for WINO)}} \\
\midrule
\multirow{2}{*}{\textbf{Dataset}} & \multirow{2}{*}{\textbf{Method}}
& \multicolumn{2}{c}{$\tau_d{=}0.6$} & \multicolumn{2}{c}{$\tau_d{=}0.7$} & \multicolumn{2}{c}{$\tau_d{=}0.8$} & \multicolumn{2}{c}{$\tau_d{=}0.9$} \\
\cmidrule(lr){3-4} \cmidrule(lr){5-6} \cmidrule(lr){7-8} \cmidrule(lr){9-10}
& & Acc & Speed & Acc & Speed & Acc & Speed & Acc & Speed \\
\midrule
\multirow{2}{*}{MATH500}
 & WINO       & 33.0 & 2.9$\times$ & 33.6 & 2.3$\times$ & 34.2 & 2.0$\times$ & 34.8 & 1.5$\times$ \\
 & \sysname{} & \textbf{33.2} & \textbf{3.3$\times$} & \textbf{35.2} & \textbf{2.9$\times$} & \textbf{35.6} & \textbf{2.2$\times$} & \textbf{36.0} & \textbf{1.7$\times$} \\
\midrule
\multirow{2}{*}{MBPP}
 & WINO       & 36.2 & 2.9$\times$ & 36.2 & 1.9$\times$ & 36.8 & 1.6$\times$ & 37.4 & 1.3$\times$ \\
 & \sysname{} & \textbf{36.2} & \textbf{3.4$\times$} & \textbf{37.8} & \textbf{3.1$\times$} & \textbf{38.6} & \textbf{2.5$\times$} & \textbf{38.6} & \textbf{2.0$\times$} \\
\midrule
\multicolumn{10}{c}{\emph{Minimum draft count $n$ (Saber)}} \\
\midrule
\multirow{2}{*}{\textbf{Dataset}} & \multirow{2}{*}{\textbf{}}
& \multicolumn{2}{c}{$n{=}4$} & \multicolumn{2}{c}{$n{=}3$} & \multicolumn{2}{c}{$n{=}2$} & \multicolumn{2}{c}{$n{=}1$} \\
\cmidrule(lr){3-4} \cmidrule(lr){5-6} \cmidrule(lr){7-8} \cmidrule(lr){9-10}
& & Acc & Speed & Acc & Speed & Acc & Speed & Acc & Speed \\
\midrule
MATH500 & Saber & 32.4 & 2.5$\times$ & 34.2 & 1.9$\times$ & 34.8 & 1.8$\times$ & 35.0 & 1.4$\times$ \\
MBPP    & Saber & 35.4 & 2.4$\times$ & 37.2 & 2.7$\times$ & 37.8 & 2.1$\times$ & 37.8 & 1.8$\times$ \\
\bottomrule[1pt]
\end{tabular}
\caption{Controlled ablation on drafting strength with LLaDA-Ins-8B; these are not the configurations used throughout the main comparison. For WINO and \sysname{}, we vary the drafting threshold $\tau_d$, keeping WINO's verification threshold fixed at $\tau_v{=}0.9$. Saber's confidence threshold is computed dynamically and cannot be fixed manually, so we instead vary the minimum draft count $n$ exposed by its official implementation. The two panels therefore use method-specific control variables and are not row-aligned.}
\label{tab:schedule_sweep}
\end{table*}

\textbf{Drafting-strength sweep.}
Table~\ref{tab:schedule_sweep} is a controlled ablation on drafting strength, not a record of the configurations used throughout the main comparison. In the top panel, we vary the drafting threshold $\tau_d$ for WINO and \sysname{}; for WINO, the verification threshold remains fixed at $\tau_v{=}0.9$, isolating the effect of drafting aggressiveness. Saber instead computes its confidence threshold dynamically from previously unmasked tokens, so the same fixed-threshold sweep is not available. We therefore vary $n$ from Saber's official implementation, which sets the minimum number of tokens drafted at each step, and report the results in a separate panel. Consequently, the two panels are diagnostic sweeps over method-specific controls and their columns are not matched configurations. \sysname{} outperforms WINO throughout the shared $\tau_d$ range on both datasets (for example, on MBPP at $\tau_d{=}0.8$, \sysname{} reaches $38.6$ versus WINO's $36.8$). Across the respective sweeps, the overall trend remains stable, indicating that the benefits of the ATC and orthogonal probe are not tied to a single drafting strength.

\section{Theoretical Analysis}

\subsection{Exponential Suppression of False Anchors}
\label{app:theory}
In this section, we provide a formal justification for the effectiveness of the Anchor Token Cache (ATC). We model the decoding process as a signal-plus-noise system to demonstrate that the temporal consistency requirement (Eq.~\eqref{eq:anchor_definition}) exponentially suppresses the selection of incorrect tokens.

We consider a specific masked position $i$ with ground-truth token $x_{0}^{i}$. Let $w \in V$ be an incorrect candidate token ($w \neq x_{0}^{i}$). We posit the following assumptions regarding the model's behavior and the decoding environment:

\begin{enumerate}
    \item \textbf{Semantic Capability:} We assume the model is semantically capable, meaning that in a noise-free (ideal) context, the intrinsic confidence for the ground truth $x_{0}^{i}$ is strictly higher than that for the incorrect candidate $w$. That is, there exists a positive \textit{semantic margin} $\Delta_{\text{sem}}$:
    \begin{equation}
        p^*(x_{0}^{i}) - p^*(w) = \Delta_{\text{sem}} > 0,
    \end{equation}
    where $p^*(\cdot)$ denotes the intrinsic, denoised probability of a token.

    \item \textbf{i.i.d. Noise:} The observed probability $p_{\theta}(x_{0}^{i}=v|x_{s})$ at any step $s$ is the sum of the intrinsic prior $p^*(v)$ and a stochastic noise term $\epsilon_{s,v}$. We further assume that the differential noise
    $\eta_{s} \triangleq \epsilon_{s,x_{0}^{i}} - \epsilon_{s,w}$ is symmetric around $0$ and i.i.d.\ across decoding steps $s$ (this holds, e.g., when $\epsilon_{s,x_{0}^{i}} \perp \epsilon_{s,w}$ with each marginally symmetric, or when the noise vector is jointly sign-symmetric).

\end{enumerate}

We now prove Proposition~\ref{prop:false_anchor}.

\begin{proof}
\textbf{Step 1: Condition for Single-Step Error.}
For the incorrect token $w$ to be identified as the prediction $\hat{a}^{i}_{s}$ at step $s$, it must achieve the highest probability among all vocabulary items. A necessary condition for this is that $w$ must specifically surpass the ground truth $x_{0}^{i}$:
\begin{equation}
    \hat{a}^{i}_{s} = w \implies p_{\theta}(x_{0}^{i}=w|x_{s}) > p_{\theta}(x_{0}^{i}=x_{0}^{i}|x_{s}).
\end{equation}
Using the noise decomposition from Assumption 2 ($p_{\theta} = p^* + \epsilon$), this inequality becomes:
\begin{equation}
    p^*(w) + \epsilon_{s,w} > p^*(x_{0}^{i}) + \epsilon_{s,x_{0}^{i}}.
\end{equation}
We define the \textit{differential noise} at step $s$ as $\eta_{s} \triangleq  \epsilon_{s,x_{0}^{i}} - \epsilon_{s,w}$. Rearranging the inequality yields: $\eta_{s} < p^*(w) - p^*(x_{0}^{i}) = -\Delta_{\text{sem}}$.
This means that the error occurs only when the differential noise $\eta_{s}$ is sufficiently negative to overcome the positive semantic margin $\Delta_{\text{sem}}$. Then, substituting the semantic margin $\Delta_{\text{sem}}$ from Assumption 1, the condition for a ranking reversal in a single step is $-\eta_{s} > \Delta_{\text{sem}}$.

\textbf{Step 2: Probability of Ranking Reversal.}
Let $q$ be the probability that the noise fluctuation is sufficient to overcome the semantic margin $q \triangleq \mathbb{P}(\eta_{s} > \Delta_{\text{sem}})$. Since the noise distribution is assumed to be symmetric, the probability of the differential noise falling into the left tail is equal to that of the right tail: $q \triangleq \mathbb{P}(\eta_{s} < -\Delta_{\text{sem}}) = \mathbb{P}(\eta_{s} > \Delta_{\text{sem}}).$ Consequently, $q < 0.5$ (and typically $q \ll 0.5$ for strong models).

\textbf{Step 3: Temporal Consistency and Exponential Decay.}
According to Eq.~\eqref{eq:anchor_definition}, the event that position $i$ is identified as an anchor with the incorrect token $w$ requires that $\hat{a}^{i}_{s} = w$ for all $s \in \{t-k+1, \dots, t\}$. By Step~1, each event $\{\hat{a}^{i}_{s} = w\}$ is contained in $\{\eta_{s} < -\Delta_{\text{sem}}\}$, so we may bound the joint probability by the joint probability of the $\eta$-events, which factorizes under the i.i.d.\ assumption:

\begin{equation}
\begin{aligned}
    \mathbb{P}&\!\left(\hat{a}^{i}_{t-k+1} {=} \cdots {=} \hat{a}^{i}_{t} = w\right) \\
    &\le \mathbb{P}\!\left(\eta_{t-k+1} < -\Delta_{\text{sem}}, \dots, \eta_{t} < -\Delta_{\text{sem}}\right) \\
    &= \prod_{j=0}^{k-1} \mathbb{P}(\eta_{t-j} < -\Delta_{\text{sem}}) \\
    &= q^k.
\end{aligned}
\end{equation}

If we rewrite $q^k$ in exponential form $q^k = \exp(-k \ln(1/q))$, let $\lambda = \ln(1/q)$, and note that $q < 1 \implies \lambda > 0$, we then obtain:
\begin{equation}
    \mathbb{P}\!\left(\hat{a}^{i}_{t-k+1} {=} \cdots {=} \hat{a}^{i}_{t} = w\right) \le \exp(-\lambda k).
\end{equation}
This completes the proof, demonstrating that increasing the window size $k$ exponentially suppresses the likelihood of false anchors driven by stochastic noise.
\end{proof}

\begin{remark}
While this mechanism suppresses false positives, a correct token $x_0^i$ with $p^*(x_0^i) > \tau$ may be excluded from the ATC if a single noise spike $\epsilon_{s, x_0^i}$ pushes its top-1 prediction off $x_0^i$, or its confidence below $\tau$ so that it leaves $\mathcal{T}_t$, at any of the $k$ consecutive steps. ASRD prioritizes minimizing this ``False Positive'' (accepting a wrong anchor) over a ``False Negative'' (missing a true anchor) due to the asymmetric cost of errors in diffusion decoding: Accepting an incorrect anchor introduces ``toxic context'', which actively steers the generation of $[\text{MASK}]$ tokens toward invalid trajectories via the attention mechanism. This leads to irreversible error propagation. In comparison, a missed valid anchor merely results in a sparser guidance signal. In this case, the model defaults to its standard decoding behavior (baseline), ensuring that ASRD never performs worse than the unguided baseline.
\end{remark}

\begin{remark}
    The i.i.d. assumption serves to derive a tractable upper bound that captures the qualitative behavior of the consistency mechanism. In practice, even with temporally correlated noise, the probability of an incorrect token persisting across $k$ steps diminishes significantly, as supported by our empirical results.
\end{remark}

\subsection{Anchor-Induced Attention Bias}
\label{appen_pro2}

In this section, we provide a calculation-based discussion to illustrate how
anchor-guided embedding injection influences attention behavior and downstream
predictions. We emphasize that the following analysis is intended as a
mechanistic intuition, rather than a formal guarantee under the full Transformer
dynamics.

\paragraph{\textbf{Intuition: Anchor-Induced Attention Bias.}}
Consider a self-attention layer, where queries, keys, and values are computed as
linear projections of token representations between a mask token at position $i$ and context token at position $u$:
$q_i = W_Q \mathbf{e}^{i}_t$, $k_u = W_K \mathbf{e}^{u}_t$, and $v_u = W_V \mathbf{e}^{u}_t$.
Here, for a masked position $i$, where $\mathbf{e}^{i}_t = \mathbf{e}_{\text{mask}}$ initially, the anchor-guided generation replaces the uninformative \texttt{[MASK]} embedding $\mathbf{e}_{\text{mask}}$, with the anchor centroid $\bar{c}_t = \frac{1}{|\mathcal{I}_t|}\sum_{j \in \mathcal{I}_t} \mathbf{e}^{j}_t$ defined in Section~\ref{sec:mask_update},

\begin{equation}
\mathbf{e}'^{i}_t = (1-\gamma_t^i)\, \mathbf{e}^{i}_t + \gamma_t^i\, \bar{c}_t ,
\end{equation}

where $\mathcal{I}_t$ denotes the Anchor Token Cache (ATC) at decoding step $t$. For notational clarity, we treat the per-position magnitude $\gamma_t^i$ (defined in Section~\ref{sec:mask_update}) as a fixed scalar throughout this derivation; the entropy-adaptive form only modulates the magnitude of the induced bias per position, not its geometry. Due to the linearity of the query projection, this replacement induces a corresponding shift in the query representation. We have $q_i' = W_Q \mathbf{e}'^{i}_t = (1-\gamma_t^i)\, q_i + \gamma_t^i\, q_{\mathcal{I}_t}$, where $q_{\mathcal{I}_t} = W_Q \bar{c}_t$. Thus, the attention logit between the masked position $i$ and an arbitrary position $u$ becomes

\begin{equation}
\begin{aligned}
s_{iu}' &= \frac{(q_i')^\top k_u}{\sqrt{d}} \\
&= (1-\gamma_t^i)\frac{(q_i)^\top k_u}{\sqrt{d}} + \gamma_t^i\frac{(q_{\mathcal{I}_t})^\top k_u}{\sqrt{d}} \\
&= (1-\gamma_t^i)\, s_{iu} + \gamma_t^i\, s_{{\mathcal{I}_t}u},
\end{aligned}
\end{equation}

\noindent where $s_{iu}$ and $s_{{\mathcal{I}_t}u}$ are the attention logits if we simply replace the mixed token by original mask token or purely aggregate from anchor tokens. Therefore, although we do not assume any specific structure on $W_Q$ or $W_K$, the injected query $q_i'$ is biased toward the subspace spanned by anchor embeddings. When anchor tokens provide a more reliable semantic signal than unverified tokens, their keys tend to exhibit higher alignment with this shifted query. As a result, the additive change in attention logits is typically larger for anchor tokens than for noisy pending tokens.

More generally, the effect of anchor-guided injection extends beyond anchor
tokens themselves to arbitrary context tokens. Consider two context tokens $u_1$ and $u_2$ that are indistinguishable under the original mask query, i.e., $\langle k_{u_1}, q_{i} \rangle \approx \langle k_{u_2}, q_{i} \rangle$, but the keys' inner product with query vector on anchor tokens' embedding satisfy

\begin{equation}
\langle k_{u_1}, q_{\mathcal{I}_t} \rangle > \langle k_{u_2}, q_{\mathcal{I}_t} \rangle ,
\end{equation}

\noindent which implies that $k_{u_1}$ is closer to the collection of anchor tokens. Substituting into the interpolation formula yields
\begin{equation}
s'_{i u_1} >  s'_{i u_2}.
\end{equation}

To isolate the effect of this alignment, consider the case where the remaining
context is otherwise comparable, i.e., the original logits $s_{iu_1}$ and
$s_{iu_2}$ are of similar magnitude prior to anchor injection. Under this
condition, the difference in attention logits after injection is dominated by
the difference in $\gamma_t^i\, s_{{\mathcal{I}_t}u}$, leading to a relative increase in the
softmax attention weight for token $u_1$ compared to $u_2$.

This comparison illustrates that anchor-guided injection does not merely
increase attention to anchor tokens themselves, but more generally amplifies
the influence of context tokens whose representations are closer to the
anchor-induced direction, while suppressing contributions from less aligned or
noisy tokens. In this sense, the mechanism selectively reshapes the effective
context seen by the masked position in a direction consistent with anchor
semantics.

\end{document}